\documentclass{article}

\usepackage{PRIMEarxiv}

\usepackage[utf8]{inputenc} 
\usepackage[T1]{fontenc}    
\usepackage{url}            
\usepackage{booktabs}       
\usepackage{nicefrac}       
\usepackage{microtype}      
\usepackage{lipsum}
\usepackage{fancyhdr}       
\usepackage{graphicx}       

\usepackage{amsthm}
\usepackage{algorithm}
\usepackage{algorithmic}
\usepackage{subcaption}

\usepackage{amsmath,amssymb,amsfonts,mathtools,bm}
\usepackage{textcomp}
\usepackage{xcolor}
\usepackage{verbatim}
\usepackage{newclude}
\usepackage{appendix}

\usepackage{gensymb}
\usepackage{hyperref}
\hypersetup{hidelinks}
\usepackage[switch]{lineno}  
\usepackage{cleveref}
\usepackage{color}

\DeclareMathOperator*{\argmin}{arg\,min}
\newtheorem{theorem}{Theorem}
\newtheorem{lemma}{Lemma}
\newtheorem{claim}{Claim}
\newtheorem{remark}{Remark}

\newcommand{\vect}[1]{\ensuremath{\bm{#1}}}
\newcommand{\E}{\ensuremath{\mathbb{E}}}
\newcommand{\norm}[1]{\left\lVert#1\right\rVert}

\title{Online Decentralized Frank-Wolfe: From theoretical bound to applications in smart-building
\thanks{Supported by the Multidisciplinary Institute in Artificial Intelligence, Univ.Grenoble Alpes, France (ANR-19-P3IA-0003)} 
}

\author{
  Angan Mitra \\
  Qarnot Computing\\
  University Grenoble-Alpes \\
  France \\
  \texttt{angan.mitra@qarnot-computing.com} \\
   \And
  Nguyen Kim Thang \\
  IBISC, University of Evry \\
  University Paris-Saclay \\
  France \\
  \texttt{kimthang.nguyen@univ-evry.fr} \\
  \AND
  {Tuan-Anh Nguyen, Denis Trystram, Paul Youssef} \\
  LIG, INRIA, Grenoble INP \\
  University Grenoble-Alpes \\
  France \\
  \texttt{\{tuan-anh.nguyen,paul.youssef\}@inria.fr} \\
  \texttt{denis.trystram@imag.fr}
}

\begin{document}
\maketitle

\begin{abstract}
The design of decentralized learning algorithms is important in the fast-growing world in which data are distributed over participants with limited local computation resources and  communication. In this direction, we propose an online algorithm minimizing non-convex loss functions aggregated from individual data/models distributed over a network. We provide the theoretical performance guarantee of our algorithm and demonstrate its utility on a real life smart building.
\end{abstract}


\section{Introduction}
\label{chap:intro}
The popularity of sensors and IoT devices has the potential of generating and equivalently accumulating data in order of Zeta bytes \cite{maccarthy2018defense} annually.
High throughput, low latency, data consumption, networking dependencies are often the key metrics in designing high-performance learning algorithms under the constraint of low powered computing.
In recent times, there has been an alternate trend to process data on cloud or dump into a centralized database.
Commonly known as edge computing, the new paradigm embraces the idea of using interconnected computing nodes to reduce high bandwidth consuming data uploads, privacy preservation of data and knowledge  on the fly.

Smart building applications typically have a profound implication on environment in terms of energy savings, reduction of green house emission, etc.
Predicting the future often forms the basis of corrective actions taken by such apps and can be regarded as a predominant use-case of machine learning. 
Usually the data is generated across multiple zones from heterogeneous sensors and forms a setting of decentralized learning.
In recent times, the hardware-software interface has benefited from advances in network communication coupled with edge computing. 
Thus deploying a machine learning model in site and processing data on the fly has become a realistic alternative of sending data to a centralised data base.
Optimizing problems to maintain robust solutions under the uncertainty of future is a nice to have feature for such cyber physical systems.
Contrary to the classical train-test-deploy framework, online learning offers continual learning where during run time, a batch of sensor data has the potential to update an AI model on site. 

This work aligns with the edge computing paradigm by proposing an online and decentralized learning algorithm. 
Online learning helps better adapt to the uncertainty of the future where the data pattern continually changes over time. The designed algorithm repeatedly chooses a high-performance strategy given a set of actions compared to the best-fixed action in hindsight. Instead of having a centralized mediator, the decentralized setting promotes peer-to-peer knowledge exchanges while prohibiting data sharing between learners. Many proposed online decentralized algorithms use gradient descent-based methods to solve constraint problems. Such an approach requires projection into the constraint set that usually involves intensive computation, which is not best suited in the context of sensors and IoT. We aim to design a competitive, robust algorithm in the decentralized and online setting that has the flexibility of being projection-free.

\paragraph*{\textbf{Problem setting}} Formally, we are given a convex set $\mathcal{K} \subseteq \mathbb{R}^d$ and a set of agents connected over a network represented by a graph $G = (V, E)$ where $n  = |V|$ is the number of agents. 
At every time $1 \leq t \leq T$, each agent $i \in V$ can communicate with (and only with) its immediate neighbors, i.e., adjacent agents in $G$ and 
takes a decision $\vect{x}^{t}_{i} \in \mathcal{K}$. 
Subsequently, a batch of new data is revealed exclusively to agent $i$ and from its own batch, a non-convex cost function $f^{t}_{i}: \mathcal{K} \rightarrow \mathbb{R}$
is induced locally. Although each agent $i$ observes only function $f^{t}_{i}$, 
agent $i$ is interested in the cumulating cost $F^{t}(\cdot)$  where 
$F^{t}(\cdot) := \frac{1}{n} \sum_{j=1}^{n} f^{t}_{j}(\cdot)$. In particular, at time $t$, 
the cost of agent $i$ with the its chosen $\vect{x}^{t}_{i}$ is $F^{t}(\vect{x}^{t}_{i})$.  
The objective of each agent $i$ is to minimize the total cumulating cost $\sum_{t=1}^{T} F^{t}(\vect{x}^{t}_{i})$
via local communication with its immediate neighbors.

When the cost functions $f^{t}_{i}$ are convex, a standard measure is the \emph{regret} notion. 
An online algorithm is \emph{$R(T)$-regret} if for every agent $1 \leq i \leq n$, 
\begin{align*}
\frac{1}{T} \biggl( \sum_{t = 1}^T F^t(\vect{x}^t_i) - \min_{\vect{o} \in \mathcal{K}} \sum_{t=1}^T F^t(\vect{o}) \biggr) \leq R(T)
\end{align*}

As the cost functions in the paper are not necessarily convex, we consider a stationary measure on the quality of 
solution based on the Frank-Wolfe gap \cite{Jaggi:2013}, and that can be considered the counter-part of the regret in 
the non-convex setting. Specifically, we aim to bound the \emph{convergence gap}, for every agent $1 \leq i \leq n$:
\begin{align}	\label{eq:gap_def}
    \max_{\vect{o} \in \mathcal{K}} \frac{1}{T} \sum_{t=1}^{T} \langle \nabla F^{t}(\vect{x}^{t}_{i}), \vect{x}^{t}_{i} - \vect{o}\rangle 
    \end{align}
In the same spirit as the regret, the measure of convergence gap compares the total cost of every agent to that of the best stationary point in hindsight. 
Note that when the functions $F^{t}$ are convex, the convergence gap is always upper bounded by the regret. 
Moreover, when the problem becomes offline, i.e., all $F^{t}$ are the same, the convergence gap measures the speed of convergence to a stationary solution.

\subsection{Our contribution}

The challenge in designing robust and efficient algorithms for the problem is to resolve the following issues together: 
the uncertainty (online setting, agents observe their own loss functions only after choosing their decisions),  
the partial information (decentralized setting, agents know only its own loss functions while aiming to minimize the cumulating cost),
and the non-convexity of the loss functions. 
As a starting point, we consider the Meta Frank-Wolfe (MFW) algorithm~\cite{ChenHassani18:Online-continuous} in the (centralized, convex) online setting and the Decentralized Frank-Wolfe (DFW) algorithm \cite{WaiLafond17:Decentralized-Frank--Wolfe} in the decentralized (offline) setting. However, these algorithms work either in the online setting or in the decentralized one but not both together. The difficulty in our problem, as mentioned earlier, is to resolve all issues together. 

In the paper, we present algorithms, subtly built on MFW and DFW algorithms, that achieves the convergence gap of $O(T^{-1/2})$ and $O(T^{-1/4})$ in cases where the exact gradients or only stochastic gradients of loss functions are available, respectively.  
Note that in the former, the convergence gap of $O(T^{-1/2})$ asymptotically matches the best regret guarantee even in the centralized offline settings with convex functions.
Besides, one can convert the algorithms to be projection-free by choosing appropriate oracles used in the algorithm. This property provides a flexibility 
to apply the algorithms to different settings depending on the computing capacity of local devices. 
Our work applies to online neural network optimization amongst a group of autonomous learners.
We demonstrate the practical utility of our algorithm in a smart building application where zones mimic learners optimizing a temperature forecasting problem.
We provide a thorough analysis of our algorithms in different angles of the performance guarantee (quality of solutions), the effects of network topology and decentralization, which are predictably explained by our theoretical results.

\subsection{Related Work}
\label{sec: relatedworks}

\paragraph{Decentralized Online Optimization.}  %
Authors \cite{Yan:2013} introduced decentralized online projected subgradient descent and showed vanishing regret for convex and strongly convex functions. 
In contrast, Hosseini et al.~\cite{Hosseini:2013} extended distributed dual averaging technique to the online setting using a general regularized projection for both unconstrained and constrained optimization.
A distributed variant of online conditional gradient \cite{Hazanothers16:Introduction-to-online} was designed and analyzed in~\cite{Zhang:2017} that requires linear minimizers and uses exact gradients. 
However, computing exact gradients may be prohibitively expensive for moderately sized data and intractable when a closed-form does not exist. 
In this work, we go a step ahead in designing a distributed algorithm that uses stochastic gradient estimates and provides a better regret bound than in~\cite{Zhang:2017}.

\paragraph{Learning on the edge. } Over the year, edge computing has become an exciting alternative for cloud-based learning by processing the data closer to end devices while ensuring data confidentiality and reducing transmission. \cite{edge-gradient-descent}  proposes a distributed framework for non-i.i.d data using multiple gradient descent-based algorithms to update local models and a dedicated edge unit for global aggregation. Another popular approach is to reduce the memory size of classical machine learning models to meet edge resource constraints. \cite{decision-jungle} and \cite{pruning-rf} similarly takes this idea by building a tree-based learning framework with a considerable reduction in memory using compression and pruning. At the same time, \cite{protonn} introduce an edge-friendly version of k-nearest neighbor \cite{knn} by projecting the data into a lower-dimensional space. 
Besides traditional machine learning algorithms, adapting deep learning models to work on edge devices is an emerging research domain. In \cite{convolution-pruning,neural_pruning_dynamic}, the authors propose a pruning technique on convolutional network for faster computation while preserving the model ability. Another approach using weight quantization is proposed in \cite{weight-quantization}. The current dominant paradigm is federated learning \cite{federatedLearning2017,McMahanothers:Advances-and-Open}, where offline centralized training is performed through a star network with multiple devices connected to a central server. However, decentralized training is more efficient than centralized one when operating on networks with low bandwidth or high latency \cite{LianZhang17:Can-decentralized-algorithms,HeBian18:Cola:-Decentralized}. In this paper, we go one step further by studying arbitrary communication networks without a central coordinator and the local data (so local cost functions) evolve.

\paragraph{Thermal Profiling a Building.} 
Usually, building monitoring sensors are distributed across a building and thus acts as a scattered data lake with potentially heterogeneous patterns.
Indoor temperature is an important factor in controlling Heating Ventilation Air Conditioning systems that maintain ambient comfort within a building \cite{gupta2015distributed}.
Typically such embedded systems run in anticipatory mode where temperature prediction \cite{cai2019day} of controlled building zones helps in maintaining thermal consistency.
A multitude of factors effect the thermal profile like outdoor environment, opening/closing of windows, number of occupants, etc, which are hard to get and often rely on intrusive mechanisms to gather the data.
Researchers have utilized deep learning models ~\cite{zamora2014line} in the context of online learning of temperature, but lack the benefit of interacting with multiple similar sensors.
This study seeks to generate a thermal profile of a building by only utilizing temperature data from multiple zones of a building in order to extract patterns about thermal variation.
The proposed methodology not only processes data on the fly \cite{abdel2019data}, but also identifies meaningful topological data exchange networks that can best predict multi zonal temperature settings.

\section{Conditional Gradient based Algorithm}
\label{chap:formulation}

In this section, after introducing and recalling useful notions, we will 
first provide an algorithm for the setting with exact gradients. Subsequently, building on the salient ideas of that algorithm, 
we extend to the more realistic setting with stochastic gradients. 

\subsection{Preliminaries and Notations}
\label{sec:math_notation}
Given an undirected graph $G = (V, E)$,
the set of neighbors of an agent $i \in V$ is $N(i) := \{j \in V: (i,j) \in E\}$.  
Consider a symmetric matrix $W \in \mathbb{R}_{+}^{n \times n}$ defined as follows. 
The entry $W_{ij}$ has a value of 
\begin{align*}
    W_{ij} = \begin{cases}
            \dfrac{1}{1+\max\{d_i, d_j\}} & \text{if $(i,j) \in E$}\\
            0 &  \text{if $(i,j) \not\in E$,$i \neq j$}\\
            1 - \sum_{j \in N(i)} W_{ij} & \text{if $i=j$}
        \end{cases}
\end{align*}
where $d_i = |N(i)|$, the degree of vertex $i$.
In fact, the matrix $W$ is doubly stochastic, i.e $W \vect{1} = W^T \vect{1} = \vect{1}$ and so it inherits several useful properties of 
doubly stochastic matrices. We use boldface letter e.g $\vect{x}$ to represent vectors. We denote $\vect{x}^t_i$ as the decision vector of agent $i$ at time step $t$. We suppose that the constraint set $\mathcal{K}$ is a bounded convex set with diameters $D = \sup_{\vect{x}, \vect{y} \in \mathcal{K}} \| \vect{x} - \vect{y} \|$. 

A function $f$ is \emph{$\beta$-smooth} if for all $\vect{x}, \vect{y} \in \mathcal{K}$ :
\begin{align*}
     f(\vect{y}) \leq f(\vect{x}) + \langle \nabla f(\vect{x}), \vect{y}-\vect{x} \rangle + \frac{\beta}{2}\|\vect{y}-\vect{x}\|^2 
\end{align*} or equivalently $\| \nabla f(\vect{x}) - \nabla f(\vect{y})\| \leq \beta \| \vect{x} - \vect{y} \|$. Also, we say a function $f$ is \emph{$G$-Lipschitz} if for all $\vect{x}, \vect{y} \in \mathcal{K}$
\begin{align*}
    \|f(\vect{x}) - f(\vect{y}) \| \leq G \|\vect{x} - \vect{y}\|
\end{align*}

In our algorithm, we make use of linear optimization oracles where its role is to resolve an online linear optimization problem given a feedback function and a constraint set. 
Specifically, in the online linear optimization problem, at every time $1 \leq t \leq T$, one has to select $\vect{u}^{t} \in \mathcal{K}$. 
Subsequently, the adversary reveals a vector $\vect{d}^{t}$ and feedbacks the cost function $\langle \cdot , \vect{d}^t \rangle$.
The objective is to minimize the regret, i.e., 
$\frac{1}{T} \bigl( \sum_{t=1}^{T} \langle \vect{u}^{t}, \vect{d}^t \rangle - \min_{\vect{u}^{*} \in \mathcal{K}} \sum_{t=1}^{T} \langle \vect{u}^{*}, \vect{d}^t \rangle \bigr)$. 
Several algorithms \cite{Hazanothers16:Introduction-to-online} provide an optimal regret bound of $\mathcal{R}^T = O(1/\sqrt{T})$ for the online linear optimization problem.
These algorithms include the online gradient descent algorithm or the follow-the-perturbed-leader algorithm (projection-free). One can pick one of such algorithms 
to be an oracle resolving the online linear optimization problem.

\subsection{An Algorithm with Exact Gradients}
\label{sec:exact}
Assume that the exact gradients of the loss functions $f^{t}_{i}$ are available (or can be computed). 
The high-level idea of the algorithm is the following. In the algorithm, at every time $t$, each agent $i$ executes $L$ steps of the Frank-Wolfe algorithm 
where every update vector (for iterations $1 \leq \ell \leq L$ where the parameter $L$ will be chosen later) is constructed by 
combining the outputs of linear optimization oracles $\mathcal{O}_{j,\ell}$
and the current vectors of its neighbors $j \in N(i)$.  During this execution, a set of feasible solutions $\{\vect{x}_{i,\ell}^{t}: 1 \leq \ell \leq L\}$ is computed.
The solution $\vect{x}_{i}^{t}$ for each agent $1 \leq i \leq n$ is then chosen uniformly at random among $\{\vect{x}_{i,\ell}^{t}: 1 \leq \ell \leq L\}$.
Subsequently, after communicating and aggregating the information related to functions $f_{j}^{t}$
for $j \in N(i)$, the algorithm computes a vector $\vect{d}^{t}_{i,\ell}$ and feedbacks $\langle \vect{d}^t_{i,\ell}, \cdot \rangle$ as the cost function at time $t$ to the oracle $\mathcal{O}_{i,\ell}$ for $1 \leq \ell \leq L$.
The vectors $\vect{d}^{t}_{i,\ell}$'s are subtly built so that it captures step-by-step more and more information on the cumulating cost functions.    
The formal description is given in Algorithm~\ref{algo:online-dist-FW} and a detailed proof of \Cref{thm:gap} is given in \cite{tuan_anh_nguyen_2022_6435193}

\begin{algorithm}[ht!]
\begin{flushleft}
\textbf{Input}:  A convex set $\mathcal{K}$, 
	a time horizon $T$, a parameter $L$, online linear optimization oracles $\mathcal{O}_{i,1}, \ldots, \mathcal{O}_{i,L}$ for each agent $1 \leq i \leq n$, 
	step sizes $\eta_\ell \in (0, 1)$ for all $1 \leq \ell \leq L$
\end{flushleft}
\begin{algorithmic}[1]
\FOR {$t = 1$ to $T$}	 		
	\FOR{every agent $1 \leq i \leq n$}	%
		\STATE Initialize arbitrarily $\vect{x}^t_{i,1} \in \mathcal{K}$ 
		\FOR{$1 \leq \ell \leq L$}
			\STATE Let $\vect{v}^{t}_{i,\ell}$ be the output of oracle $\mathcal{O}_{i,\ell}$ at time step $t$.
			\STATE Send $\vect{x}^{t}_{i,\ell}$ to all neighbours $N(i)$
			\STATE \label{alg:y} 
				Once receiving $\vect{x}^{t}_{j,\ell}$ from all neighbours $j \in N(i)$, 
				set $\vect{y}^{t}_{i,\ell} \gets \sum_{j} W_{ij} \vect{x}^{t}_{j,\ell}$.
			\STATE \label{alg:x} Compute $\vect{x}^{t}_{i,\ell+1} \gets (1 - \eta_{\ell}) \vect{y}^{t}_{i,\ell} + \eta_{\ell} \vect{v}^{t}_{i,\ell}$.
		\ENDFOR
		\STATE Choose $\vect{x}^{t}_{i} \gets \vect{x}^{t}_{i,\ell}$ for $1 \leq \ell \leq L$ with probability $\frac{1}{L}$ and play $\vect{x}^t_{i}$
		\STATE Receive function $f^{t}_{i}$ 
		\STATE Set $\vect{g}^{t}_{i,1} \gets \nabla f^{t}_{i}(\vect{x}^{t}_{i,1})$
			\FOR{$1 \leq \ell \leq L$}
				\STATE Send $\vect{g}^{t}_{i,\ell}$ to all neighbours $N(i)$.
				\STATE After receiving $\vect{g}^{t}_{j,\ell}$ from all neighbours $j \in N(i)$, compute
					$\vect{d}^{t}_{i,\ell} \gets  \sum_{j \in N(i)} W_{ij} \vect{g}^{t}_{j,\ell}$
					and
					$\vect{g}^{t}_{i,\ell + 1} \gets \bigl( \nabla f^{t}_{i}(\vect{x}^t_{i,\ell+1}) 
						-  \nabla f^{t}_{i}(\vect{x}^{t}_{i,\ell}) \bigr) + \vect{d}^{t}_{i,\ell}$.
				\STATE Feedback function $\langle \vect{d}^{t}_{i,\ell}, \cdot \rangle$ 
				to oracles $\mathcal{O}_{i,\ell}$. (The cost of the oracle $\mathcal{O}_{i,\ell}$ at time $t$ is 
				$\langle \vect{d}^{t}_{i,\ell}, \vect{v}^{t}_{i,\ell}  \rangle$.)
			\ENDFOR
	\ENDFOR
\ENDFOR
\end{algorithmic}
\caption{Online Decentralized algorithm}
\label{algo:online-dist-FW}
\end{algorithm}

\begin{theorem}
\label{thm:gap}
Let $\mathcal{K}$ be a convex set with diameter D. Assume that functions $F^{t}$ (possibly non convex) are $\beta$-smooth and G-Lipschitz for every $1 \leq t \leq T$. 
Then, by choosing the step size $\eta_{\ell} = \min\left(1, \frac{A}{\ell^{\alpha}}\right)$ for some $A \geq 0$ and $\alpha \in (0,1)$, 
Algorithm \ref{algo:online-dist-FW} guarantees that for all $1 \leq i \leq n$:
\setlength{\textfloatsep}{0pt}
    \begin{align*}
      \max_{\vect{o} \in \mathcal{K}}\frac{1}{T} &\sum_{t=1}^{T} \E_{\vect{x}^t_i} \bigl [\langle \nabla F^{t}(\vect{x}^t_{i}), \vect{x}^t_{i} - \vect{o}\rangle \bigr] 
      \leq O \left( \frac{GDA^{-1}}{L^{1-\alpha}}  
         + \frac{AD^2 \beta/2}{L^{\alpha}(1-\alpha)} + \mathcal{R}^{T} \right) 
    \end{align*}
where $\mathcal{R}^T$ is the regret of online linear minimization oracles. Choosing $L=T$, $\alpha = 1/2$ and oracles as gradient descent or follow-the-perturbed-leader with regret $\mathcal{R}^T =
O\bigl(T^{-1/2}\bigr)$, we obtain the gap convergence rate of $O\bigl(T^{-1/2}\bigr)$.
\end{theorem}

\subsection{Algorithm with Stochastic Gradients}

We extend the previous algorithm to the setting of stochastic gradients estimates. As only stochastic gradient estimates are available, 
we use a variance reduction technique in order to upgrade~\Cref{algo:online-dist-FW} to its stochastic version (Algorithm 2).
The difference between the two algorithms is stochastic gradient estimation and an additional step for variance reduction. 
After making the decision, the agent receives an unbiased gradient to perform updates and communication to obtain stochastic estimates $\widetilde{\vect{g}}^t_{i,\ell}$
and $\widetilde{\vect{d}}^t_{i,\ell}$ of $\vect{g}^t_{i,\ell}$ and $\vect{d}^t_{i,\ell}$, respectively. (Note that the stochastic variables are denoted by the same letter as its exact counterpart with an additional tilde symbol.) Then the agent uses Step \ref{step:var-red} in  \Cref{algo:online-dist-FW-stoc}  to get the reduced variance version $\widetilde{\vect{a}}^t_{i,\ell}$ of $\widetilde{\vect{d}}^t_{i,\ell}$. The function $\langle \widetilde{\vect{a}}^t_{i,\ell}, \cdot \rangle$ is then feedbacked to the oracle.

The formal description is given in \Cref{algo:online-dist-FW-stoc} in which all previous steps are the same as \Cref{algo:online-dist-FW}
and the additional variance reduction step is marked in red. A detailed proof of \Cref{thm:stoc:version2} can be found in \cite{tuan_anh_nguyen_2022_6435193}. 

\begin{algorithm}
		\qquad $\ldots$
\begin{algorithmic}[1]
\setcounter{ALC@line}{11}
		\STATE Receive function $f^{t}_{i}$ and an unbiased gradient estimate $\widetilde {\nabla} f^{t}_{i}$
		\STATE Set $\widetilde{\vect{g}}^{t}_{i,1} \gets \widetilde{\nabla} f^{t}_{i}(\vect{x}^{t}_{i,1})$
		\FOR{$1 \leq \ell \leq L$}
				\STATE Send $\widetilde{\vect{g}}^{t}_{i,\ell}$ to all neighbours $N(i)$.
				\STATE After receiving $\widetilde{\vect{g}}^{t}_{j,\ell}$ from $j \in N(i)$, compute
					$\widetilde{\vect{d}}^{t}_{i,\ell} \gets  \sum_{j \in N(i)} W_{ij} \widetilde{\vect{g}}^{t}_{j,\ell}$ and set $\widetilde{\vect{g}}^{t}_{i,\ell + 1} \gets \bigl( \widetilde{\nabla} f^{t}_{i}(\vect{x}^t_{i,\ell+1}) 
						-  \widetilde{\nabla} f^{t}_{i}(\vect{x}^{t}_{i,\ell}) \bigr) + \widetilde{\vect{d}}^{t}_{i,\ell}$.
				\STATE \label{step:var-red} {\color{red} $\widetilde{\vect{a}}^t_{i, \ell} \gets (1 - \rho_\ell) \cdot \widetilde{\vect{a}}^t_{i, \ell-1} + \rho_\ell \cdot \widetilde{\vect{d}}^{t}_{i,\ell}$}.
				\STATE Feedback function $\langle \widetilde{\vect{a}}^{t}_{i,\ell}, \cdot \rangle$ 
				to oracles $\mathcal{O}_{i,\ell}$. (The cost of the oracle $\mathcal{O}_{i,\ell}$ at time $t$ is 
				$\langle \widetilde{\vect{a}}^{t}_{i,\ell}, \vect{v}^{t}_{i,\ell}  \rangle$.)
		\ENDFOR
\end{algorithmic}
\caption{Stochastic online decentralized algorithm}
\label{algo:online-dist-FW-stoc}
\end{algorithm}

\begin{theorem}
\label{thm:stoc:version2}
Let $\mathcal{K}$ be a convex set with diameter $D$. Assume that for every $1 \leq t \leq T$. 
\begin{enumerate}
	\item functions $f^{t}_{i}$ are $\beta$-smooth and $G$-Lipschitz, 
	\item the gradient estimates are unbiased with bounded variance $\sigma^{2}$,
	\item the gradient estimates are Lipschitz.
\end{enumerate}
Then, choosing the step-sizes $\eta_\ell = \min \{1, \frac{A}{\ell^{3/4}}\}$ for some $A \geq 0$, we have for all $1 \leq i \leq n$, 
    \begin{align*}
        & \max_{\vect{o} \in \mathcal{K}} \E \Bigl[ \frac{1}{T} \sum_{t=1}^T \E_{\vect{x}_i^t}\bigl [\langle \nabla F_t\left( \vect{x}_i^t \right), \vect{x}_i^t - \vect{o} \rangle \bigr] \Bigl] 
         \leq O \left( \frac{DG + 2ADQ^{1/2}}{L^{1/4}}
            + \frac{2AD^2\beta}{L^{3/4}} + \mathcal{R}^T \right) 
    \end{align*}
Choosing $L=T$ and oracles with regret $\mathcal{R}^T =
O\left(T^{-1/2}\right)$, we obtain the convergence gap of $O\left( T^{-1/4} \right).$

\end{theorem}

\section{Experiments}
\label{chap:experiments }

The data-set used for experimentation comes from a 7 storey building with 24 sensor equipped zones \cite{data-bangkok}.The zone-wise knowledge exchange happens through the edges of an undirected graph of $n$ nodes participating in the learning process.
For every round $t$, each node $i$ receives a batch $\mathcal{B}^t_i$ of 32 time-series sequences corresponding to a look-back period 13 timestep to predict the temperature of the next timestep. We extract the data from March $7^{th}$ to April $20^{th}$ for training, set $L$ equal to 360, $\alpha = 0.95$ and $A = 1$. 
A min-max scaler is used to normalize the data and we apply a rolling window with stride 1 on the original time series.
Each node is embedded with a model built from a two-layers long-short-time-memory (LSTM) network followed by a fully connected layer. Denote the output of the model $i$ for a data sequence $b$ at time $t$  by $\hat{y}^t_{i,b}$ and its ground truth by $y^t_{i,b}$. Consider the $\ell_1$ loss as the objective function :
\begin{linenomath}
    \begin{align*}
        \mathcal{L}(\hat{y}^t_{i,b},y^t_{i,b}) =
        \begin{cases}
      \dfrac{(\hat{y}^t_{i,b}-y^t_{i,b})^2}{2} \quad \text{if } |\hat{y}^t_{i,b}-y^t_{i,b} | \leq 1 \\
      |\hat{y}^t_{i,b}-y^t_{i,b}| - \frac{1}{2} \quad \text{otherwise.}\\
    \end{cases}
    \end{align*}
\end{linenomath}

Consider the constraint set $\mathcal{K} = \{\vect{x} \in \mathbb{R}^{d}, \|\vect{x}\|_1 \leq r\}$, where $\vect{x}$ is the model's weight, $d$ its dimension and $r=1$. The (normalized) loss incurred by the data of agent $i$ is 
$
\frac{1}{|\mathcal{B}^t_i|}\sum_{b \in \mathcal{B}^t_i}\mathcal{L}(\hat{y}^t_{i,b},y^t_{i,b}).
$
The global loss function incurred by the overall data is
\begin{align*}
F^t(\vect{x})
= \frac{1}{|\cup_{i=1}^{n} \mathcal{B}^t_i|}\sum_{b \in \cup_{i=1}^{n} \mathcal{B}^t_i}\mathcal{L}(\hat{y}^t_{i,b},y^t_{i,b}),
\end{align*}
that can be written as $F^t(\vect{x}) = \frac{1}{n} \sum_{i=1}^n f^t_i(\vect{x})$ where
$
f^t_i(\vect{x})= \frac{1}{|\mathcal{B}^t_i|} \sum_{b \in \mathcal{B}^t_i}\mathcal{L}(\hat{y}^t_{i,b},y^t_{i,b}).
$
Note that the non-convexity here is due to the non-convexity of $\hat{y}^t_{i,b}$ as a function of $\vect{x}^t_i$. 
In the following section, if not specify otherwise, we call \emph{loss} the temporal average of the global loss function $F^t$ defined as $\frac{1}{T}\sum_{t=1}^{T} F^t$.

\subsection{Prediction Performance }

Figures \ref{fig:loss-multiple-size} and \ref{fig:gap-multiple-size} show the loss and gap values for different network sizes. 
The implementation justifies our theoretical results about the convergence of the gap. Besides, we also observe the convergence of loss value, an expected implication of the gap convergence. 
We set $M$ the number of prediction points between the $21^{st}$ and $24^{th}$ of April and $n$ the number of zones within one configuration. We use the mean absolute error (MAE $=\frac{1}{nM}\sum_{i=1}^n \sum_{m=1}^{M} |\hat{y}_{i,m} - y_{i,m}| $) and mean square error (MSE $= \frac{1}{nM}\sum_{i=1}^n \sum_{m=1}^{M} \left( \hat{y}_{i,m} - y_{i,m} \right)^2 $) as a measure between the prediction and the ground truth.
We observe that increasing nodes in a network does not always lead to better online performance.
In-fact, a 7 node configuration achieves the lowest MSE (0.65) and MAE (0.78) for floors 6 and 7.
We see a 40 $\%$ drop in MSE and 20 $\%$ reduction in MAE for floor 6 zonal models when 3 extra peers from floor 7 joined the group.
We observe 19 $\%$ and 25 $\%$ increase in MSE and MAE values by adding zonal nodes from floor 7 to a 10 node group.
This can be best argued by the fact that the top floor of a building has a non identical thermal variation with the rest of the storeys. 

\begin{figure}[h]
\centering
\begin{subfigure}{0.5\textwidth}
  \centering
  \includegraphics[width=\textwidth]{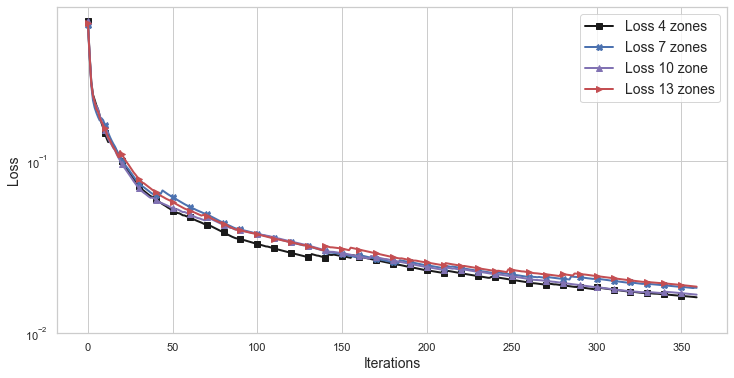}
  \caption{Loss Value}
  \label{fig:loss-multiple-size}
\end{subfigure}%
\hfill
\begin{subfigure}{.5\textwidth}
  \centering
  \includegraphics[width=\textwidth]{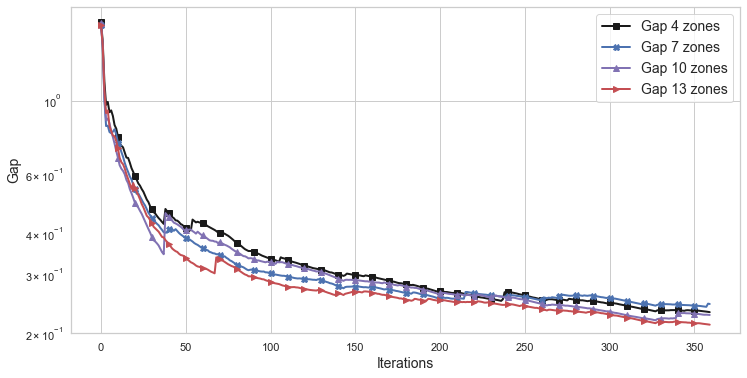}
  \caption{Gap Value}
  \label{fig:gap-multiple-size}
\end{subfigure}
\caption{Loss and Gap values of different network size on complete topology     \textit{(Plot on log-scale)}}
\label{fig:fig}
\end{figure}

\subsection{Effect of Network Topology}

We study the effect of topology in learning for a 7 node configuration with a complete, cycle and line graph containing 28, 7 and 6 edges respectively and with 13 nodes having 78,13 and 12 edges respectively.
For both 7 (Table \ref{table:temp8}) and 13 (Table \ref{table:temp13}) node configurations, we observe that the complete graph yields the least amount of prediction error, mean absolute error $\in [0.66, 1.3] \degree C$. 
However we note the peculiarity that the line graph can perform better than a cycle graph and has roughly a 10 $\%$ error margin compared to the complete configuration.

\begin{table}[h]
\centering
\begin{subtable}[t]{0.45\textwidth}
\begin{center}
\begin{tabular}{llllll}
\toprule
Topology & Metric    & Mean & Var & Max   & Min \\ 
\midrule
cycle    & MAE   & 1.09 & 0.48  & 1.80 &  0.56 \\ 
cycle    & MSE  &  0.78 & 0.21  & 1.09 &  0.52  \\ 
complete & MAE   & \textbf{0.77} & \textbf{0.38} & \textbf{1.47} & 0.27 \\ 
complete & MSE  & \textbf{0.64} & \textbf{0.20} & \textbf{1.04} & 0.39 \\ 
line     & MAE   & 0.81 & 0.53 & 1.95 & \textbf{0.24} \\ 
line     & MSE   & 0.66 & 0.28 & 1.26 & \textbf{0.34} \\ 
\bottomrule
\end{tabular}
\caption{Impact of Topology on 7 learners configuration.}
\label{table:temp8}
\end{center}
\end{subtable}
\hspace{0.8cm}
\begin{subtable}[t]{0.45\textwidth}
\begin{center}
\begin{tabular}{llllll}
\toprule
Topology & Metric    & Mean & Var & Max   & Min \\
\midrule
cycle    & MAE   & 1.51 & 1.46 & 6.16 & 0.36 \\ 
cycle    & MSE   & 0.94  & 0.38  & 1.90  & 0.48 \\ 
complete & MAE  & \textbf{1.26 }& \textbf{0.82} & 3.64 & \textbf{0.32} \\ 
complete & MSE   & \textbf{0.85} & \textbf{0.27} & \textbf{1.50} & \textbf{0.42} \\ 
line     & MAE   & 1.38 & 0.91 & 3.17 & 0.50 \\ 
line     & MSE   & 0.90 & 0.35 & \textbf{1.66} & 0.49 \\ 
\bottomrule
\end{tabular}
\caption{Impact of Topology on 13 learners configuration.}
\label{table:temp13}
\end{center}
\end{subtable}
\caption{Temperature forecasting performances on different network topologies}
\label{tab:main}
\end{table}

\subsection{Effect of Decentralization}

\begin{figure}[h]
    \centering
    \includegraphics[scale=0.4]{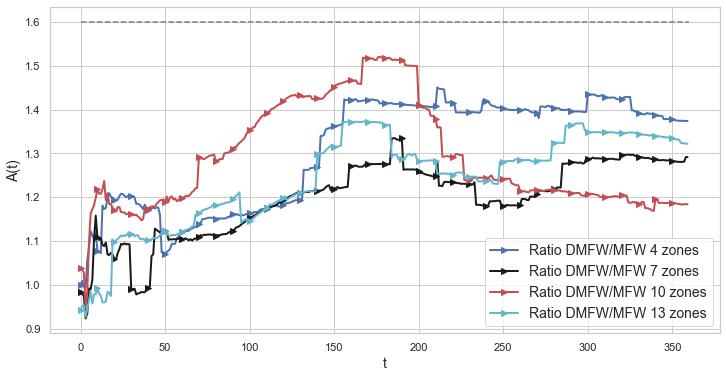}
    \caption{Loss ratio of decentralized and centralized Meta Frank-Wolfe on different network size.}
    \label{fig:ratio}
\end{figure}

We are interested in understanding the role of decentralization in terms of accuracy of zonal learners.
Let $L_{MFW}(t)$ be the loss from Meta Frank Wolfe (MFW) at time $t$.
The approximation ratio $A(t) = \frac{L_{DMFW}(t) }{L_{MFW}(t)}$ at time $t$ represents how worse is our decentralized version compared to a centralized optimization. 
$A(t) \leq B_{max}$ will mean our algorithm performs no worse than $B_{max}$ times of the MFW.
On figure \ref{fig:ratio}, we plot the ratio $A(t)$ for a 13 node network and show that $A(t) \leq 1.4$. 
The 7 node network has the closest approximation bounded by $1.35$ which can be explained by earlier insights on performance accuracy.  
We notice that the 10 node network performs worse till $t = 200$ and after $t \geq 250 $ or 21 hours, the approximation ratio becomes close to centralised version with less than 20 $\%$ error.

\section{Concluding remarks}
\label{chap:conclusion}

We proposed an online algorithm minimizing non-convex loss functions aggregated from local data distributed over a network. We showed the bounds of the  convergence gap in both exact and stochastic gradient settings. 
In complement to the theoretical analysis, we run experiments on a real-life smart building data-set.
The results make our offerings valuable for learning in distributed settings.

\newpage

\bibliographystyle{unsrt}  
\bibliography{ref}  

\begin{thebibliography}{10}

\bibitem{maccarthy2018defense}
Mark MacCarthy.
\newblock In defense of big data analytics.
\newblock {\em The Cambridge Handbook of Consumer Privacy}, pages 47--78, 2018.

\bibitem{Jaggi:2013}
Martin Jaggi.
\newblock Revisiting {Frank-Wolfe}: Projection-free sparse convex optimization.
\newblock In {\em Proceedings of the 30th International Conference on Machine
  Learning}, 2013.

\bibitem{ChenHassani18:Online-continuous}
Lin Chen, Hamed Hassani, and Amin Karbasi.
\newblock Online continuous submodular maximization.
\newblock In {\em Proc. 21st International Conference on Artificial
  Intelligence and Statistics (AISTAT)}, 2018.

\bibitem{WaiLafond17:Decentralized-Frank--Wolfe}
H.~Wai, J.~Lafond, A.~Scaglione, and E.~Moulines.
\newblock Decentralized frank--wolfe algorithm for convex and nonconvex
  problems.
\newblock {\em IEEE Transactions on Automatic Control}, 62(11):5522--5537,
  2017.

\bibitem{Yan:2013}
F.~{Yan}, S.~{Sundaram}, S.~V.~N. {Vishwanathan}, and Y.~{Qi}.
\newblock Distributed autonomous online learning: Regrets and intrinsic
  privacy-preserving properties.
\newblock {\em IEEE Transactions on Knowledge and Data Engineering},
  25(11):2483--2493, 2013.

\bibitem{Hosseini:2013}
S.~{Hosseini}, A.~{Chapman}, and M.~{Mesbahi}.
\newblock Online distributed optimization via dual averaging.
\newblock In {\em 52nd IEEE Conference on Decision and Control}, pages
  1484--1489, 2013.

\bibitem{Hazanothers16:Introduction-to-online}
Elad Hazan.
\newblock Introduction to online convex optimization.
\newblock {\em Foundations and Trends{\textregistered} in Optimization},
  2(3-4):157--325, 2016.

\bibitem{Zhang:2017}
W.~Zhang, P.~Zhao, W.~Zhu, S.C.V. Hoi, and T.~Zhang.
\newblock Projection-free distributed online learning in networks.
\newblock In {\em Proceedings of the 34th International Conference on Machine
  Learning}, pages 4054--4062, 2017.

\bibitem{edge-gradient-descent}
Shiqiang Wang, Tiffany Tuor, Theodoros Salonidis, Kin~K. Leung, Christian
  Makaya, Ting He, and Kevin Chan.
\newblock When edge meets learning: Adaptive control for resource-constrained
  distributed machine learning.
\newblock In {\em IEEE INFOCOM 2018 - IEEE Conference on Computer
  Communications}, pages 63--71, 2018.

\bibitem{decision-jungle}
Jamie Shotton, Toby Sharp, Pushmeet Kohli, Sebastian Nowozin, John Winn, and
  Antonio Criminisi.
\newblock Decision jungles: Compact and rich models for classification.
\newblock In C.J. Burges, L.~Bottou, M.~Welling, Z.~Ghahramani, and K.Q.
  Weinberger, editors, {\em Advances in Neural Information Processing Systems},
  volume~26. Curran Associates, Inc., 2013.

\bibitem{pruning-rf}
Feng Nan, Joseph Wang, and Venkatesh Saligrama.
\newblock Pruning random forests for prediction on a budget.
\newblock In D.~Lee, M.~Sugiyama, U.~Luxburg, I.~Guyon, and R.~Garnett,
  editors, {\em Advances in Neural Information Processing Systems}, volume~29.
  Curran Associates, Inc., 2016.

\bibitem{protonn}
Chirag Gupta, Arun~Sai Suggala, Ankit Goyal, Harsha~Vardhan Simhadri, Bhargavi
  Paranjape, Ashish Kumar, Saurabh Goyal, Raghavendra Udupa, Manik Varma, and
  Prateek Jain.
\newblock {P}roto{NN}: Compressed and accurate k{NN} for resource-scarce
  devices.
\newblock In Doina Precup and Yee~Whye Teh, editors, {\em Proceedings of the
  34th International Conference on Machine Learning}, volume~70 of {\em
  Proceedings of Machine Learning Research}, pages 1331--1340. PMLR, 06--11 Aug
  2017.

\bibitem{knn}
Thomas~M. Cover and Peter~E. Hart.
\newblock Nearest neighbor pattern classification.
\newblock {\em IEEE Trans. Inf. Theory}, 13(1):21--27, 1967.

\bibitem{convolution-pruning}
Zhang Chiliang, Hu~Tao, Guan Yingda, and Ye~Zuochang.
\newblock Accelerating convolutional neural networks with dynamic channel
  pruning.
\newblock In {\em 2019 Data Compression Conference (DCC)}, pages 563--563,
  2019.

\bibitem{neural_pruning_dynamic}
Ji~Lin, Yongming Rao, Jiwen Lu, and Jie Zhou.
\newblock Runtime neural pruning.
\newblock In I.~Guyon, U.~Von Luxburg, S.~Bengio, H.~Wallach, R.~Fergus,
  S.~Vishwanathan, and R.~Garnett, editors, {\em Advances in Neural Information
  Processing Systems}, volume~30. Curran Associates, Inc., 2017.

\bibitem{weight-quantization}
Taylor Simons and Dah-Jye Lee.
\newblock A review of binarized neural networks.
\newblock {\em Electronics}, 8(6), 2019.

\bibitem{federatedLearning2017}
Brendan McMahan and Daniel Ramage.
\newblock Collaborative machine learning without centralized training data.
\newblock {\em Google Research Blog}, 3, 2017.

\bibitem{McMahanothers:Advances-and-Open}
Peter Kairouz, H.~Brendan McMahan, Brendan Avent, Aurélien Bellet, et~al.
\newblock Advances and open problems in federated learning.
\newblock {\em Foundations and Trends{\textregistered} in Machine Learning},
  14(1), 2021.

\bibitem{LianZhang17:Can-decentralized-algorithms}
Xiangru Lian, Ce~Zhang, Huan Zhang, Cho-Jui Hsieh, Wei Zhang, and Ji~Liu.
\newblock Can decentralized algorithms outperform centralized algorithms? a
  case study for decentralized parallel stochastic gradient descent.
\newblock In {\em Advances in Neural Information Processing Systems}, pages
  5330--5340, 2017.

\bibitem{HeBian18:Cola:-Decentralized}
Lie He, An~Bian, and Martin Jaggi.
\newblock Cola: Decentralized linear learning.
\newblock In {\em Advances in Neural Information Processing Systems}, pages
  4536--4546, 2018.

\bibitem{gupta2015distributed}
Santosh~K Gupta, Koushik Kar, Sandipan Mishra, and John~T Wen.
\newblock Distributed consensus algorithms for collaborative temperature
  control in smart buildings.
\newblock In {\em 2015 American Control Conference (ACC)}, pages 5758--5763.
  IEEE, 2015.

\bibitem{cai2019day}
Mengmeng Cai, Manisa Pipattanasomporn, and Saifur Rahman.
\newblock Day-ahead building-level load forecasts using deep learning vs.
  traditional time-series techniques.
\newblock {\em Applied energy}, 236:1078--1088, 2019.

\bibitem{zamora2014line}
Fransisco Zamora-Martinez, Pablo Romeu, Pablo Botella-Rocamora, and Juan Pardo.
\newblock On-line learning of indoor temperature forecasting models towards
  energy efficiency.
\newblock {\em Energy and Buildings}, 83:162--172, 2014.

\bibitem{abdel2019data}
Hamzah Abdel-Aziz and Xenofon Koutsoukos.
\newblock Data-driven online learning and reachability analysis of stochastic
  hybrid systems for smart buildings.
\newblock {\em Cyber-Physical Systems}, 5(1):41--64, 2019.

\bibitem{tuan_anh_nguyen_2022_6435193}
Angan Mitra, Nguyen~Kim Thang, Tuan-Anh Nguyen, Denis Trystram, and Paul
  Youssef.
\newblock {Online Decentralized Frank-Wolfe: From theoretical bound to
  applications in smart-building}, April 2022.

\bibitem{zhang20_quantized:2020}
Mingrui Zhang, Lin Chen, Aryan Mokhtari, Hamed Hassani, and Amin Karbasi.
\newblock Quantized frank-wolfe: Faster optimization, lower communication, and
  projection free.
\newblock In {\em Proceedings of the Twenty Third International Conference on
  Artificial Intelligence and Statistics}, volume 108, pages 3696--3706, 2020.

\end{thebibliography}

\appendix

\section*{Supplementary File for Decentralized Meta Frank-Wolfe for Online Non-Convex Optimization} 

\section{Proof of \Cref{thm:gap}} 


\setcounter{lemma}{0}
\begin{lemma}[\cite{WaiLafond17:Decentralized-Frank--Wolfe}, Lemmas 1 and 2]	
\label{lem:convergence}
Assume that functions $f^{t}_{j}$'s are $\beta$-smooth, G-Lipschitz that is, $\| \nabla f^{t}_{j}\|  \leq G$
for every $1 \leq t \leq T$ and every $1 \leq j \leq n$ and the diameter of $\mathcal{K}$ is $D$.
Then, there exists a constant $\ell_{0}$ such that for every $1 \leq \ell \leq L+1$, 
\begin{linenomath}
\begin{align*}
	\Delta p_{\ell} := \max_{t=1}^{T} \max_{i = 1}^{n} \| \vect{y}^{t}_{i,\ell} - \overline{\vect{x}}^{t}_{\ell} \| 
		&\leq \frac{C_{p}}{\ell} 
\end{align*}
\begin{align*}
	\Delta d_{\ell} := \max_{t=1}^{T} \max_{i = 1}^{n} \| \vect{d}^{t}_{i,\ell} - \frac{1}{n}\sum_{j=1}^n \nabla f^t_j (\vect{y}^t_{j,\ell})  \| 
		&\leq \frac{C_{d}}{\ell}
\end{align*}
\end{linenomath}
where 
$C_{p}=\ell_{0}\sqrt{n} D$ and $C_{d}=\sqrt{n} \cdot \max $ $\left \{ 2 \left ( \frac{\ell_{0} \sqrt{n} D}{\ell} +D \right )\beta ;
|\lambda_{2}(W)| \ell_{0} \left ( \frac{\beta D}{1 - |\lambda_{2}(W)|}+ G \right) \right\}
$ where $\lambda_2(W)$ is the second largest eigenvalue of $W$.
\end{lemma}

\setcounter{lemma}{2}
\begin{lemma}
\label{lmm:avg}
For every $1 \leq t \leq T$ and $1 \leq \ell \leq L$, it holds that 
\begin{equation}
  \overline{\vect{x}}^t_{\ell+1} - \overline{\vect{x}}^t_{\ell} = \eta_{\ell} \left( \frac{1}{n}\sum_{i=1}^{n} \vect{v}^t_{i,\ell} - \overline{\vect{x}}^t_{\ell}\right)
\end{equation}
\end{lemma} 

\begin{proof}
\begin{linenomath}
\begin{align*}
\overline{\vect{x}}_{\ell+1}^t 
&= \frac{1}{n} \sum_{i=1}^{n} \vect{x}^{t}_{i,\ell + 1} \tag{Definition of $\overline{\vect{x}}_{\ell+1}^t$}
\\
&= \frac{1}{n} \sum_{i=1}^{n} \left ((1-\eta_{\ell}) \vect{y}^{t}_{i,\ell} + \eta_{\ell} \vect{v}_{i,\ell}^t \right )  \tag{Definition of $\vect{x}_{i,\ell}^t$} 
\\
&= \frac{1}{n} \sum_{i=1}^{n} \left [ (1-\eta_{\ell}) \left ( \sum_{j=1}^n \vect{W}_{ij} \vect{x}_{j,\ell}^t  \right ) + \eta_{\ell} \vect{v}_{i,\ell}^t \right ] \tag{Definition of $\vect{y}_{i,\ell}^t$}
\\
&= (1 - \eta_{\ell}) \frac{1}{n} \sum_{i=1}^n \left [ \sum_{j=1}^n \vect{W}_{ij}\vect{x}_{j,\ell}^t \right ] + \frac{1}{n}\eta_{\ell} \sum_{i=1}^n \vect{v}_{i,\ell}^t
\\
&= (1 - \eta_{\ell}) \frac{1}{n} \sum_{j=1}^n \left [ \vect{x}_{j,\ell}^t \sum_{i=1}^n \vect{W}_{ij} \right ] + \frac{1}{n} \eta_{\ell} \sum_{i=1}^n \vect{v}_{i,\ell}^t
\\
&= (1 - \eta_{\ell}) \frac{1}{n} \sum_{j=1}^n \vect{x}_{j,\ell}^t + \frac{1}{n} \eta_{\ell} \sum_{i=1}^n \vect{v}_{i,\ell}^t \tag{$\sum_{i=1}^n W_{ij} =1$ for every $j$}
\\
&= (1 - \eta_{\ell}) \overline{\vect{x}}_{\ell}^t + \frac{1}{n} \eta_{\ell} \sum_{i=1}^n \vect{v}_{i,\ell}^t
\\
&= \overline{\vect{x}}_{\ell}^t + \eta_{\ell} \left ( \frac{1}{n} \sum_{i=1}^n \vect{v}_{i,\ell}^t - \overline{\vect{x}}_{\ell}^t \right )
\end{align*}
\end{linenomath}
where we use a property of $W$ which is $\sum_{i = 1}^{n} W_{ij} = 1$ for every $j$.
\end{proof}

\setcounter{lemma}{1}

\begin{lemma}
\label{lmm:final_step}
For every $1 \leq i \leq n$ and $1 \leq \ell \leq L$, it holds that
\begin{align*}
    \max_{\vect{o} \in \mathcal{K}}\langle \nabla F^{t} (\vect{x}^t_{i,\ell}), \vect{x}^t_{i,\ell} - \vect{o} \rangle 
     \leq \max_{\vect{o} \in \mathcal{K}}\langle \nabla F^{t} (\overline{\vect{x}}^t_{\ell}), \overline{\vect{x}}^t_{\ell} - \vect{o} \rangle
    + \left(\beta D + G \right)C_p \frac{\log L}{L}.
\end{align*}
\end{lemma}
\begin{proof} 
Fix $1 \leq i \leq n$ and $1 \leq \ell \leq L$. We have
    \begin{align*}
        \langle \nabla F^{t} (\vect{x}^t_{i,\ell}), \vect{x}^t_{i,\ell} - \vect{o} \rangle 
        & = \langle \nabla F^{t} (\overline{\vect{x}}^t_{\ell}), \overline{\vect{x}}^t_{\ell} - \vect{o} \rangle + \langle \nabla F^{t} (\vect{x}^t_{i,\ell}) - \nabla F^{t} (\overline{\vect{x}}^t_\ell), \overline{\vect{x}}^t_\ell - \vect{o} \rangle + \langle \nabla F^{t} (\vect{x}^t_{i,\ell}), \vect{x}^t_{i,\ell} - \overline{\vect{x}}^t_\ell \rangle
    \end{align*}
Therefore,
    \begin{align*}
        \max_{\vect{o} \in \mathcal{K}}\langle \nabla F^{t} (\vect{x}^t_{i,\ell}), \vect{x}^t_{i,\ell} - \vect{o} \rangle 
        & \leq \max_{\vect{o} \in \mathcal{K}}\langle \nabla F^{t} (\overline{\vect{x}}^t_\ell), \overline{\vect{x}}^t_\ell - \vect{o} \rangle + \max_{\vect{o} \in \mathcal{K}} \langle \nabla F (\vect{x}^t_{i,\ell}) - \nabla F^{t} (\overline{\vect{x}}^t_\ell), \overline{\vect{x}}^t_\ell - \vect{o} \rangle \\
        	& \qquad + \langle \nabla F^{t} (\vect{x}^t_{i,\ell}), \vect{x}^t_{i,\ell} - \overline{\vect{x}}^t_\ell \rangle \\
        & \leq \max_{\vect{o} \in \mathcal{K}}\langle \nabla F^{t} (\overline{\vect{x}}^t_\ell), \overline{\vect{x}}^t_\ell - \vect{o} \rangle  + (\beta D + G) \E{\|\vect{x}^t_{i,\ell} - \overline{\vect{x}}^t_\ell\|} \\
        & \leq \max_{\vect{o} \in \mathcal{K}}\langle \nabla F^{t} (\overline{\vect{x}}^t_\ell), \overline{\vect{x}}^t_\ell - \vect{o} \rangle + (\beta D + G)C_p \frac{\log L}{L}.
    \end{align*}
\end{proof}

\setcounter{theorem}{0}

\begin{theorem}
\label{thm:gap}
Let $\mathcal{K}$ be a convex set with diameter D. Assume that functions $F^{t}$ (possibly non convex) are $\beta$-smooth and G-Lipschitz for every t. With the choice of step size $\eta_{\ell} = \min\left(1, \frac{A}{\ell^{\alpha}}\right)$ where $A \in \mathbb{R_{+}}$ and $\alpha \in (0,1)$. 
Then, Algorithm 1 guarantees that for all $1 \leq i \leq n$:
    \begin{align*}
       \max_{\vect{o} \in \mathcal{K}}\frac{1}{T} \sum_{t=1}^{T} \E_{\vect{x}^t_i} \bigl [\langle \nabla F^{t}(\vect{x}^t_{i}), \vect{x}^t_{i} - \vect{o}\rangle \bigr] \nonumber 
       & \leq \frac{GDA^{-1}}{L^{1-\alpha}}  
         + \frac{AD^2 \beta/2}{L^{\alpha}(1-\alpha)} + O \left(\mathcal{R}^{T}\right) \notag \\
        & \quad + \left(\left( \beta C_p + C_d \right)D + \left(\beta D + G \right)C_p \right)\frac{\log L}{L}
    \end{align*}
where $\mathcal{R}^T$ is the regret of online linear minimization oracles
and 
$C_{p}=\ell_{0}\sqrt{n} D$ and $C_{d}=\sqrt{n} \cdot \max $ $\left \{ 2 \left ( \frac{\ell_{0} \sqrt{n} D}{\ell} +D \right )\beta ;
|\lambda_{2}(W)| \ell_{0} \left ( \frac{\beta D}{1 - |\lambda_{2}(W)|}+ G \right) \right\}
$ where $\lambda_2(W)$ is the second largest eigenvalue of matrix $W$ 
($C_p$ and $C_d$ are already defined in \Cref{lem:convergence}).

Choosing $L=T$, $\alpha = 1/2$ and oracles as gradient descent or follow-the-perturbed-leader with regret $\mathcal{R}^T =
O\bigl(T^{-1/2}\bigr)$, we obtain the gap convergence rate of $O\bigl(T^{-1/2}\bigr)$.
\end{theorem}
\begin{proof}
By $\beta$-smoothness, $\forall \ell \in \{1, \cdots, L\}$: 
\begin{align}	\label{tk:smth}
    F^{t}\left( \overline{\vect{x}}^{t}_{\ell+1} \right) - F^{t} \left( \overline{\vect{x}}^{t}_{\ell} \right) 
    &\leq \langle \nabla F^{t} \left( \overline{\vect{x}}^t_{\ell} \right), \overline{\vect{x}}^t_{\ell+1} - \overline{\vect{x}}^t_{\ell} \rangle 
    	+ \frac{\beta}{2} \left \| \overline{\vect{x}}^t_{\ell+1} - \overline{\vect{x}}^t_{\ell} \right\|^{2}
\end{align} 
Using Lemma~\ref{lmm:avg}, the inner product in~(\ref{tk:smth}) can be re-written as : 
    \begin{align}	\label{tk:inner_beta}
        \left \langle \nabla F^{t} \left( \overline{\vect{x}}^t_{\ell} \right), \overline{\vect{x}}^t_{\ell+1} - \overline{\vect{x}}^t_{\ell} \right \rangle 	\notag 
        &= \eta_{\ell} \left \langle \nabla F^{t} \left( \overline{\vect{x}}^t_{\ell} \right), \frac{1}{n}\sum_{i=1}^{n} \vect{v}^t_{i,\ell} - \overline{\vect{x}}^t_{\ell} \right \rangle 	\notag \\
        &= \eta_{\ell} \left \langle \nabla F^{t} \left( \overline{\vect{x}}^t_{\ell} \right),  \frac{1}{n} \biggl(\sum_{i=1}^{n} \vect{v}^t_{i,\ell} - n \cdot \overline{\vect{x}}^t_{\ell} \biggr) \right \rangle 	\notag \\
        &=  \frac{\eta_{\ell}}{n}\sum_{i=1}^{n} \left \langle \nabla F^{t} \left( \overline{\vect{x}}^t_{\ell} \right), \vect{v}^t_{i,\ell} - \overline{\vect{x}}^t_{\ell} \right \rangle
    \end{align}
Let $\vect{o}^t_{\ell}$ be such that 
$
        \vect{o}^t_{\ell} \in \argmin_{\vect{o} \in \mathcal{K}}\langle \nabla F(\overline{\vect{x}}^t_{\ell}),\vect{o} \rangle
$. Hence, 
    \begin{align*}
     \mathcal{G}^t_{\ell} 
     = \max_{\vect{o} \in \mathcal{K}}\langle \nabla F(\overline{\vect{x}}^t_{\ell}), \overline{\vect{x}}^t_{\ell} - \vect{o}\rangle 
     = \langle \nabla F(\overline{\vect{x}}^t_{\ell}), \overline{\vect{x}}^t_{\ell} - \vect{o}^t_{\ell}\rangle
    \end{align*} 
We have :
    \begin{align*}
        &\left \langle \nabla F^{t} \left( \overline{\vect{x}}^t_{\ell} \right) , \vect{v}^t_{i,\ell} - \overline{\vect{x}}^t_{\ell} \right \rangle \\
        & \quad =\langle \nabla F^{t} \left( \overline{\vect{x}}^t_{\ell} \right) - \vect{d}^t_{i,\ell}, \vect{v}^t_{i,\ell} - \vect{o}^t_{\ell} \rangle 
         	+ \langle \vect{d}^t_{i,\ell}, \vect{v}^t_{i,\ell} - \vect{o}^t_{\ell} \rangle 
         	+ \langle \nabla F^{t} \left( \overline{\vect{x}}^t_{\ell} \right), \vect{o}^t_{\ell} - \overline{\vect{x}}^t_{\ell} \rangle \\
        & \quad \leq \|\nabla F^{t} \left(\overline{\vect{x}}^t_{\ell} \right) - \vect{d}^t_{i, \ell} \| \|\vect{v}^t_{i,\ell} - \vect{o}^t_{\ell}\| 
        		 + \langle \vect{d}^t_{i,\ell}, \vect{v}^t_{i,\ell} - \vect{o}^t_{\ell} \rangle 
         	     + \langle \nabla F^{t} \left( \overline{\vect{x}}^t_{\ell} \right), \vect{o}^t_{\ell} - \overline{\vect{x}}^t_{\ell} \rangle \\
         & \quad \leq \|\nabla F^{t} \left(\overline{\vect{x}}^t_{\ell} \right) - \vect{d}^t_{i, \ell} \|D 
            + \langle \vect{d}^t_{i,\ell}, \vect{v}^t_{i,\ell} - \vect{o}^t_{\ell} \rangle 
         	+ \langle \nabla F^{t} \left( \overline{\vect{x}}^t_{\ell} \right), \vect{o}^t_{\ell} - \overline{\vect{x}}^t_{\ell} \rangle. 
    \end{align*}
where we use Cauchy-Schwarz in the first inequality.  
Using \cref{lem:convergence} and $\beta$-smoothness of $F^t$,
    \begin{align*}
        &\left\| \nabla F^{t} \left(\overline{\vect{x}}^t_{\ell} \right) - \vect{d}^t_{i, \ell} \right\|\\
        &\quad  \leq \| \nabla F^{t} \left(\overline{\vect{x}}^t_{\ell} \right) - \frac{1}{n}\sum_{i=1}^n \nabla f^t_i (\vect{y}^t_{i,\ell}) \| 
        		+ \lVert\frac{1}{n}\sum_{i=1}^n \nabla f^t_i (\vect{y}^t_{i,\ell}) - \vect{d}^t_{i, \ell} \lVert \\
        &\quad  \leq \| \frac{1}{n}\sum_{i=1}^n \nabla f^{t}_{i} \left(\overline{\vect{x}}^t_{\ell} \right) - \frac{1}{n}\sum_{i=1}^n \nabla f^t_i (\vect{y}^t_{i,\ell}) \| 
        		+ \lVert\frac{1}{n}\sum_{i=1}^n \nabla f^t_i (\vect{y}^t_{i,\ell}) - \vect{d}^t_{i, \ell} \lVert \\
        &\quad \leq  \frac{1}{n}\sum_{i=1}^n \|\nabla f^{t}_{i} \left(\overline{\vect{x}}^t_{\ell} \right) -  \nabla f^t_i (\vect{y}^t_{i,\ell}) \| 
        		 + \lVert\frac{1}{n}\sum_{i=1}^n \nabla f^t_i (\vect{y}^t_{i,\ell}) - \vect{d}^t_{i, \ell} \lVert \\
        &\quad  \leq \frac{\beta}{n} \sum_{i=1}^n \| \overline{\vect{x}}^t_{\ell} - \vect{y}^t_{i, \ell} \| \tag{by $\beta$ smoothness}
        		+ \|\frac{1}{n}\sum_{i=1}^n \nabla f^t_i (\vect{y}^t_{i,\ell}) - \vect{d}^t_{i, \ell} \| \\
        &\quad  \leq \frac{\beta C_p + C_d}{\ell}  \tag{by Lemma \ref{lem:convergence}}
    \end{align*}
Thus,
    \begin{align*}
        & \left \langle \nabla F^{t} \left( \overline{\vect{x}}^t_{\ell} \right) , \vect{v}^t_{i,\ell} - \overline{\vect{x}}^t_{\ell} \right \rangle 
        \leq \left( \frac{\beta C_p + C_d}{\ell}\right)D + \langle \vect{d}^t_{i,\ell}, \vect{v}^t_{i,\ell} - \vect{o}^t_{\ell} \rangle - \mathcal{G}^t_{\ell} 
    \end{align*}
Upper bound  the right hand side of \cref{tk:inner_beta} by the above inequality, we have :
    \begin{align}	
    \label{tk:inner_beta2}
        \left \langle \nabla F^{t} \left( \overline{\vect{x}}^t_{\ell} \right), \overline{\vect{x}}^t_{\ell+1} - \overline{\vect{x}}^t_{\ell} \right \rangle 
         \leq \eta_{\ell}\frac{\left( \beta C_p + C_d\right)D }{\ell}
          +  \frac{\eta_{\ell}}{n} \sum_{i=1}^{n}\langle \vect{d}^t_{i,\ell}, \vect{v}^t_{i,\ell} - \vect{o}^t_{\ell} \rangle - \eta_{\ell}\mathcal{G}^t_{\ell} 
    \end{align}
Combining \cref{tk:smth} with \cref{tk:inner_beta2} and re-arrange the terms, as $\eta_{\ell} = \frac{A}{\ell^{\alpha}}$, we have : 
    \begin{align}	\label{tk:gap_eta_ell}
         \eta_{\ell} \mathcal{G}^t_{\ell} 
        & \leq F^{t} \left( \overline{\vect{x}}^t_{\ell} \right) - F^{t} \left( \overline{\vect{x}}^t_{\ell+1} \right) + \eta_{\ell} \frac{(\beta C_p + C_d)D}{\ell} \notag \\
        & \quad + \frac{\eta_{\ell}}{n} \sum_{i=1}^{n}\langle \vect{d}^t_{i,\ell}, \vect{v}^t_{i,\ell} - \vect{a}^t_{\ell} \rangle + \eta^2_{\ell}D^2 \frac{\beta}{2} 
    \end{align}
Dividing by $\eta_{\ell}$ yields :
    \begin{align}	\label{tk:gap_ell}
        \mathcal{G}^t_{\ell} 
        %
        &\leq \frac{L^{\alpha}}{A} \left( F^{t} \left( \overline{\vect{x}}^t_{\ell} \right) - F^{t} \left( \overline{\vect{x}}^t_{\ell+1} \right) \right) + \frac{(\beta C_p + C_d)D}{\ell} \notag \\
        &\quad + \frac{1}{n} \sum_{i=1}^{n}\langle \vect{d}^t_{i,\ell}, \vect{v}^t_{i,\ell} - \vect{o}^t_{\ell} \rangle + \eta_{\ell}D^2 \frac{\beta}{2}
    \end{align}
Let $\mathcal{G}^t$ be a random variable such that $\mathcal{G}^t = \mathcal{G}^t_{\ell}$ with probability $\frac{1}{L}$. 
We are now bounding $\E_{\overline{\vect{x}}^t} \left[ \mathcal{G}^t \right]$.
By \cref{tk:gap_ell}, using the definition of $\eta_{\ell} = \frac{A}{\ell^{\alpha}}$ and $G$-Lipschitz property of $F$, we have 
%
    \begin{align}
        \E_{\overline{\vect{x}}^t} \bigl[\mathcal{G}^t] = \frac{1}{L}\sum_{\ell=1}^{L} \mathcal{G}^t_{\ell}
        & \leq \frac{L^{\alpha}GDA^{-1}}{L} 
        + \frac{\left(\beta C_p + C_d \right)D}{L}\sum_{\ell=1}^{L} \frac{1}{\ell} 
        + \frac{1}{nL} \sum_{\ell=1}^{L}\sum_{i=1}^{n} \langle \vect{d}^t_{i,\ell}, \vect{v}^t_{i,\ell} - \vect{o}^t_{\ell} \rangle \notag \\
        & \quad + \frac{AD^2 \beta/2}{L} \sum_{\ell=1}^{L} \frac{1}{\ell^{\alpha}} \notag \\
        & \leq  \frac{GDA^{-1}}{L^{1-\alpha}} 
        + \left(\beta C_p + C_d \right)D \frac{\log L}{L}
        + \frac{1}{nL} \sum_{\ell=1}^{L}\sum_{i=1}^{n} \langle \vect{d}^t_{i,\ell}, \vect{v}^t_{i,\ell} - \vect{o}^t_{\ell} \rangle \notag \\  
        & \quad + \frac{AD^2 \beta/2}{L} \frac{L^{1-\alpha}}{1-\alpha} \notag \\ 
        & \leq \frac{GDA^{-1}}{L^{1-\alpha}} 
        + \left( \beta C_p + C_d \right)D\frac{\log L}{L}
        + \frac{1}{nL}\sum_{\ell=1}^{L}\sum_{i=1}^{n}\langle \vect{d}^t_{i,\ell}, \vect{v}^t_{i,\ell} - \vect{o}^t_{\ell} \rangle \notag \\
        & \quad + \frac{AD^2 \beta/2}{L^{\alpha}(1-\alpha)} \label{tk:exp_gap}
    \end{align}
%
%
%
Summing the above inequality for $1 \leq t \leq T$ and
note that $\frac{1}{T}\sum_{t=1}^{T}\langle \vect{d}^t_{i,\ell}, \vect{v}^t_{i,\ell} - \vect{o}^t_{\ell} \rangle$ 
is the regret of the oracle $\mathcal{O}_{i}$, we get
    \begin{align}		\label{tk:case1_finalbound}
        \frac{1}{T}\sum_{t=1}^{T} \E_{\overline{\vect{x}}^t} \bigl[\mathcal{G}^t]  
          \leq \frac{GDA^{-1}}{L^{1-\alpha}} + O \left(\mathcal{R}^{T}\right) 
            + \left( \beta C_p + C_d \right)D\frac{\log L}{L} 
            + \frac{AD^2 \beta/2}{L^{\alpha}(1-\alpha)}
    \end{align} 
By uniformly random choice of $\vect{x}^{t}_{i}$ (over all $\vect{x}^{t}_{i,\ell}$ for $1 \leq \ell \leq L$) in the algorithm, we have

\begin{align}	
\label{tk:connection}
    \frac{1}{T}\sum_{t=1}^{T} \E_{\vect{x}^t_i} \bigl [ \max_{\vect{o} \in \mathcal{K}}\langle \nabla F^{t}(\vect{x}^t_{i}), \vect{x}^t_{i} - \vect{o}\rangle \bigr] 
    & \leq \frac{1}{T}\sum_{t=1}^{T} \frac{1}{L} \sum_{\ell=1}^{L} \bigl [ \max_{\vect{o} \in \mathcal{K}}\langle \nabla F^{t}(\vect{x}^t_{i,\ell}), \vect{x}^t_{i,\ell} - \vect{o}\rangle \bigr] \notag \\
    & \leq \frac{1}{T}\sum_{t=1}^{T} \frac{1}{L} \sum_{\ell=1}^{L} \biggl[ \max_{\vect{o} \in \mathcal{K}}
    		\langle \nabla F^{t} (\overline{\vect{x}}^t_{\ell}), \overline{\vect{x}}^t_{\ell} - \vect{o} \rangle \notag 
    +  \left(\beta D + G \right)C_p \frac{\log L}{L}
    		\biggr]  \tag{\cref{lmm:final_step}} \\
    &= \frac{1}{T}\sum_{t=1}^{T} \E_{\overline{\vect{x}}^t} \left[\mathcal{G}^t\right] +  \left(\beta D + G \right) C_p \frac{\log L}{L} 	\label{tk:connection}
\end{align}
where the last equality holds since 
\begin{align*}
    \E_{\overline{\vect{x}}^t} \left[\mathcal{G}^t\right] = \E_{\overline{\vect{x}}^t} {\left[\max_{\vect{o} \in \mathcal{K}}\langle \nabla F_t (\overline{\vect{x}}^t_{\ell}), \overline{\vect{x}}^t_{\ell} - \vect{o} \rangle \right]}
\end{align*}
Using Jensen's inequality, we have :
    \begin{align} \label{tk:jensen1}
        \max_{\vect{o} \in \mathcal{K}} \frac{1}{T} &\sum_{t=1}^{T}  \E_{\vect{x}^t_i} \bigl [\langle \nabla F^{t}(\vect{x}^t_{i}), \vect{x}^t_{i} - \vect{o}\rangle \bigr] 
            \leq \frac{1}{T}\sum_{t=1}^{T} \E_{\vect{x}^t_i} \bigl [ \max_{\vect{o} \in \mathcal{K}}\langle \nabla F^{t}(\vect{x}^t_{i}), \vect{x}^t_{i} - \vect{o}\rangle \bigr] 
    \end{align}
The theorem follows \cref{tk:jensen1}, \cref{tk:connection} and \cref{tk:case1_finalbound} and setting $L=T$. 
\end{proof}

\section{An Algorithm with Stochastic Gradient Estimates}
\setcounter{algorithm}{1}
\begin{algorithm}
\begin{flushleft}
\textbf{Input}:  A convex set $\mathcal{K}$, 
	a time horizon $T$, a parameter $L$, online linear optimization oracles $\mathcal{O}_{i,1}, \ldots, \mathcal{O}_{i,L}$ for each player $1 \leq i \leq n$, 
	step sizes $\eta_\ell \in (0, 1)$ for all $1 \leq \ell \leq L$
\end{flushleft}
\begin{algorithmic}[1]
\STATE Initialize linear optimizing oracle $\mathcal{O}_{i,\ell}$ for all $1 \leq \ell \leq L$
\FOR {$t = 1$ to $T$}	 		
	\FOR{every agent $1 \leq i \leq n$}	%
		\STATE Initialize arbitrarily $\vect{x}^t_{i,1} \in \mathcal{K}$ and set $\widetilde{\vect{{a}}}^t_{i,0} \gets \vect{0}$ 
		\FOR{$1 \leq \ell \leq L$}
			\STATE Let $\vect{v}^{t}_{i,\ell}$ be the output of oracle $\mathcal{O}_{i,\ell}$ at time step $t$.
			\STATE Send $\vect{x}^{t}_{i,\ell}$ to all neighbours $N(i)$
			\STATE \label{alg:y} 
				Once receiving $\vect{x}^{t}_{j,\ell}$ from all neighbours $j \in N(i)$, 
				set $\vect{y}^{t}_{i,\ell} \gets \sum_{j} W_{ij} \vect{x}^{t}_{j,\ell}$.
			\STATE \label{alg:x} Compute $\vect{x}^{t}_{i,\ell+1} \gets (1 - \eta_{\ell}) \vect{y}^{t}_{i,\ell} + \eta_{\ell} \vect{v}^{t}_{i,\ell}$.
		\ENDFOR
		\STATE Choose $\vect{x}^{t}_{i} \gets \vect{x}^{t}_{i,\ell}$ for $1 \leq \ell \leq L$ with probability $\frac{1}{L}$ and play $\vect{x}^t_{i}$
		\STATE Receive function $f^{t}_{i}$ and an unbiased gradient estimate $\widetilde {\nabla} f^{t}_{i}$
		\STATE Set $\widetilde{\vect{g}}^{t}_{i,1} \gets \widetilde{\nabla} f^{t}_{i}(\vect{x}^{t}_{i,1})$
			\FOR{$1 \leq \ell \leq L$}
				\STATE Send $\widetilde{\vect{g}}^{t}_{i,\ell}$ to all neighbours $N(i)$.
				\STATE After receiving $\widetilde{\vect{g}}^{t}_{j,\ell}$ from all neighbours $j \in N(i)$, compute
					$\widetilde{\vect{d}}^{t}_{i,\ell} \gets  \sum_{j \in N(i)} W_{ij} \widetilde{\vect{g}}^{t}_{j,\ell}$ and set $\widetilde{\vect{g}}^{t}_{i,\ell + 1} \gets \bigl( \widetilde{\nabla} f^{t}_{i}(\vect{x}^t_{i,\ell+1}) 
						-  \widetilde{\nabla} f^{t}_{i}(\vect{x}^{t}_{i,\ell}) \bigr) + \widetilde{\vect{d}}^{t}_{i,\ell}$.
				\STATE $ \widetilde{\vect{a}}^t_{i, \ell} \gets (1 - \rho_\ell) \cdot \widetilde{\vect{a}}^t_{i, \ell-1} + \rho_\ell \cdot \widetilde{\vect{d}}^{t}_{i,\ell}$.
				\STATE Feedback function $\langle \widetilde{\vect{a}}^{t}_{i,\ell}, \cdot \rangle$ 
				to oracles $\mathcal{O}_{i,\ell}$. (The cost of the oracle $\mathcal{O}_{i,\ell}$ at time $t$ is 
				$\langle \widetilde{\vect{a}}^{t}_{i,\ell}, \vect{v}^{t}_{i,\ell}  \rangle$.)
			\ENDFOR
	\ENDFOR
\ENDFOR
\end{algorithmic}
\caption{Stochastic online decentralized algorithm}
\label{algo:online-dist-FW-stoc}
\end{algorithm}
\setcounter{lemma}{4}
\begin{lemma}[Lemma 3, \cite{zhang20_quantized:2020}]
\label{lmm:red_var1}
Let $\{\vect{d}_{\ell}\}_{\ell \geq 1}$ be a sequence of points in $\mathbb{R}^n$ such that $\|\vect{d}_{\ell} - \vect{d}_{\ell-1}\| \leq \dfrac{B}{(\ell +3)^{\alpha}}$ for all $\ell \geq 1$ with fixed constant $B \geq 0$, $\alpha \in (0,1]$. Let $\{ \widetilde{\vect{d}}_{\ell} \}$ be a sequence of random variables such that $\mathbb{E} [\widetilde{\vect{d}}_{\ell}|\mathcal{H}_{\ell-1}] = \vect{d}_{\ell}$ and $\mathbb{E} \left[ \bigl \| \widetilde{\vect{d}}_{\ell} - \vect{d}_{\ell} \bigr \|^2 | \mathcal{H}_{\ell-1} \right] \leq \sigma^2$ for every $\ell \geq 1$, where $\mathcal{H}_{\ell - 1}$ is the history up to $\ell-1$. Let $\{ \widetilde{\vect{a}}_{\ell}\}_{\ell \geq 0}$ be a sequence of random variables defined recursively as  
\begin{linenomath}
    \[ \widetilde{\vect{a}}_{\ell} = (1 - \rho_{\ell}) \widetilde{\vect{a}}_{\ell-1} + \rho_{\ell} \widetilde{\vect{d}}_{\ell} \] 
\end{linenomath}
for $\ell \geq 1$ where  $\rho_{\ell} = \dfrac{2}{(\ell+3)^{2\alpha/3}}$ and 
$\widetilde{\vect{a}}_{0}$ is fixed. Then we have 

\begin{equation*}
        \E{\|\vect{d}_{\ell} - \Tilde{\vect{a}}_{\ell} \|^2} \leq \frac{Q}{(\ell + 4)^{2\alpha/3}}
\end{equation*}
where $Q = \max\{4^{2\alpha/3}\|\Tilde{\vect{a}}_0 - \vect{d}_0\|^2, 4\sigma^2 + 2B^2\}$
\end{lemma}
\setcounter{lemma}{5}
\begin{lemma}
\label{lem:forQ}
Given the assumptions of \Cref{thm:stoc:version2}, for every $1 \leq t \leq T$, $1 \leq i \leq n$ and $1 \leq \ell \leq L$, it holds that 
\begin{equation*}
    \E{\|\Tilde{\vect{d}}^t_{i,\ell} - \vect{d}^t_{i,\ell} \|^2} \leq 12\left(\Tilde{\beta}^2 + \beta^2\right) \left(2C_p + AD \right)^2 + \sigma^2
\end{equation*}
\end{lemma}
\begin{proof}
    Fix an arbitrary time $t$. For any $1 \leq i \leq n$, we have 

\begin{align}	\label{eq:claim-d-1}
\E & \left[  \| \widetilde{\vect{d}}^t_{i,\ell + 1} - \vect{d}^t_{i,\ell+1} \|^2 \right] 	\notag \\
&= \E \biggl[ \left \| \widetilde{\nabla} f^{t}_{i}(\vect{x}^{t}_{i,\ell+1}) - \widetilde{\nabla} f^{t}_{i}(\vect{x}^{t}_{i,\ell}) - ( \nabla f^{t}_{i}(\vect{x}^{t}_{i,\ell+1}) - \nabla f^{t}_{i}(\vect{x}^{t}_{i,\ell}) )+   (\widetilde{\vect{d}}^t_{i,\ell} - \vect{d}^t_{i,\ell}) \right \|^2 \biggr] 	\notag \\
&= \E \biggl[ \left \| \widetilde{\nabla} f^{t}_{i}(\vect{x}^{t}_{i,\ell+1}) - \widetilde{\nabla} f^{t}_{i}(\vect{x}^{t}_{i,\ell}) - ( \nabla f^{t}_{i}(\vect{x}^{t}_{i,\ell+1})  - \nabla f^{t}_{i}(\vect{x}^{t}_{i,\ell}) ) \right \|^{2} +   \left  \| \widetilde{\vect{d}}^t_{i,\ell} - \vect{d}^t_{i,\ell} \right \|^2 \biggr] 	\notag \\
&\leq \E \biggl[ 4 \bigl( \| \widetilde{\nabla} f^{t}_{i}(\vect{x}^{t}_{i,\ell+1}) - \widetilde{\nabla} f^{t}_{i}(\vect{x}^{t}_{i,\ell}) \|^{2}
	+ \| \nabla f^{t}_{i}(\vect{x}^{t}_{i,\ell+1})  - \nabla f^{t}_{i}(\vect{x}^{t}_{i,\ell}) \|^{2} \bigr)
	+ \left  \| \widetilde{\vect{d}}^t_{i,\ell} - \vect{d}^t_{i,\ell} \right \|^2 \biggr]	\notag \\
&\leq \E \biggl [4 (\widetilde{\beta}^{2} + \beta^{2}) \| \vect{x}^{t}_{i,\ell+1} - \vect{x}^{t}_{i,\ell}\|^{2}
	+  \left  \| \widetilde{\vect{d}}^t_{i,\ell} - \vect{d}^t_{i,\ell} \right \|^2 \biggr]
\end{align}

The second equality holds since 
$\E \bigl[ \widetilde{\nabla} f^{t}_{i}(\vect{x}^{t}_{i,\ell+1}) - \widetilde{\nabla} f^{t}_{i}(\vect{x}^{t}_{i,\ell}) - ( \nabla f^{t}_{i}(\vect{x}^{t}_{i,\ell+1})  - \nabla f^{t}_{i}(\vect{x}^{t}_{i,\ell}) ) \bigr ] = 0$. 
The first inequality follows the fact that $\|\vect{a} + \vect{b}\|^{2} \leq 4 ( \| \vect{a} \|^{2} + \| \vect{b} \|^{2}) $.
The last inequality is due to the $\beta$-Lipschitz and $\widetilde{\nabla}$-Lipschitz of $\nabla f^{t}_{i}$ 
and $ \widetilde{\nabla} f^{t}_{i}$, respectively.

Moreover, 

\begin{align}	
\| \vect{x}^{t}_{i,\ell+1} - \vect{x}^{t}_{i,\ell}\| 
&\leq \| \vect{x}^{t}_{i,\ell+1} - \overline{\vect{x}}^{t}_{\ell+1} \| +  \| \overline{\vect{x}}^{t}_{\ell+1} - \overline{\vect{x}}^{t}_{\ell} \|
	+ \| \overline{\vect{x}}^{t}_{\ell} - \vect{x}^{t}_{i,\ell} \|	\notag \\
&\leq \frac{2C_{p}}{\ell} 
	+ \| \overline{\vect{x}}^{t}_{\ell+1} - \overline{\vect{x}}^{t}_{\ell} \|  \tag{by \cref{lem:convergence}} \\
&= \frac{2C_{p}}{\ell} 
	+ \eta_{\ell} \biggl \| \frac{1}{n} \sum_{j=1}^{n} \vect{v}^{t}_{j,\ell} - \overline{\vect{x}}^{t}_{\ell} \biggr \|
 		\tag{by \cref{lmm:avg}} \\
&\leq \frac{2C_{p}}{\ell} + \eta_{\ell} D \\
&\leq \frac{2C_{p} + AD}{\ell^{3/4}}	\label{eq:claim-d-2}
\end{align}

where in the last inequality, $ \bigl \| \frac{1}{n} \sum_{j=1}^{n} \vect{v}^{t}_{j,\ell} - \overline{\vect{x}}^{t}_{\ell} \bigr \| \leq D$ 
for every $t,\ell$ since both $\frac{1}{n} \sum_{j=1}^{n} \vect{v}^{t}_{j,\ell}$ and $\overline{\vect{x}}^{t}_{\ell}$ are in $\mathcal{K}$. 
Therefore, combining \cref{eq:claim-d-1} and \cref{eq:claim-d-2}, we get 

\begin{align}	\label{eq:claim-d-rec}
\E \left[  \| \widetilde{\vect{d}}^t_{i,\ell + 1} - \vect{d}^t_{i,\ell+1} \|^2 \right] 
\leq 4 (\widetilde{\beta}^{2} + \beta^{2}) \frac{(2C_{p} + AD)^{2}}{\ell^{3/2}}
	+  \left  \| \widetilde{\vect{d}}^t_{i,\ell} - \vect{d}^t_{i,\ell} \right \|^2 
\end{align}

Applying \cref{eq:claim-d-rec} recursively on $\ell$, we deduce that

\begin{align*}
\E \left[  \| \widetilde{\vect{d}}^t_{i,\ell + 1} - \vect{d}^t_{i,\ell+1} \|^2 \right] 
&\leq 4 (\widetilde{\beta}^{2} + \beta^{2})(2C_{p} + AD)^{2} \sum_{l=1}^{\ell} \frac{1}{l^{3/2}}
	+  \left  \| \widetilde{\vect{d}}^t_{i,1} - \vect{d}^t_{i,1} \right \|^2 \\
&\leq 12 (\widetilde{\beta}^{2} + \beta^{2})(2C_{p} + AD)^{2}
	+  \left  \| \widetilde{\vect{d}}^t_{i,1} - \vect{d}^t_{i,1} \right \|^2
\end{align*}

since $ \sum_{l=1}^{\ell} \frac{1}{l^{3/2}}  \leq 3$.
Besides, for any $1 \leq i \leq n$

\begin{align*}
\E \left[  \| \widetilde{\vect{d}}^t_{i,1} - \vect{d}^t_{i,1} \|^2 \right] 
&= \E \biggl[  \biggl \|  \sum_{j} W_{ij} (\widetilde{\vect{g}}^t_{j,1} - \vect{g}^t_{j,1}) \biggr \|^2 \biggr] 
\leq \sum_{j} W_{ij} \E \biggl[  \bigl \|  \widetilde{\vect{g}}^t_{j,1} - \vect{g}^t_{j,1} \bigr \|^2 \biggr] \\
&= \sum_{j} W_{ij} \E \biggl[  \bigl \|  \widetilde{\nabla} f^{t}_{j}(\vect{x}^{t}_{j,1}) -  \nabla f^{t}_{j}(\vect{x}^{t}_{j,1}) \bigr \|^2 \biggr] 
\leq \sigma^{2}
\end{align*}

since $\sum_{j} W_{ij} = 1$. Hence, 

\[
\E \left[  \| \widetilde{\vect{d}}^t_{i,\ell + 1} - \vect{d}^t_{i,\ell+1} \|^2 \right] 
\leq  12(\widetilde{\beta}^{2} + \beta^{2})(2C_{p} + AD)^{2} + \sigma^{2}. 
\]

\end{proof}

\begin{claim}
\label{clm:B}
It holds that, 
\begin{equation*}
    \| \vect{d}^t_{i,\ell+1} - \vect{d}^t_{i, \ell} \| \leq \frac{B}{(\ell+3)^{\alpha}} 
\end{equation*}
where $B = 4C_{d} + 2\beta \left[ 2C_{p} + AD \right]$
\end{claim}
\begin{proof}
    \begin{align}
\norm{\overline{\vect{x}}_{\ell}^{t} - \overline{\vect{x}}_{\ell-1}^{t}} &= \eta_{\ell} \norm{ \left[ \frac{1}{n} \left( \sum_{j=1}^{n} \vect{v}_{j,\ell-1}^{t} \right) \right] - \overline{\vect{x}}_{\ell-1}^{t}} \tag{By \cref{lmm:avg}} \nonumber\\
&\leq \eta_{\ell} D \tag{$\frac{1}{n} \sum_{j=1}^{n} \overline{\vect{v}}_{j,\ell-1}^{t} \in \mathcal{K}$,  $\overline{\vect{x}}_{\ell-1} \in \mathcal{K}$ and $D = \sup_{x,y \in \mathcal{K}^2} \norm{x-y}$ }\nonumber\\
&= \frac{AD}{\ell}\tag{Definition of $\eta_{\ell} = \frac{A}{\ell}$} \\
\norm{\overline{\vect{x}}_{\ell}^{t} - \overline{\vect{x}}_{\ell-1}^{t}} &\leq \frac{AD}{\ell} \label{eq:distance_xbars}
\end{align}
\begin{align}
\norm{\vect{x}_{j,\ell}^{t} - \vect{x}_{j,\ell-1}^{t}} &\leq \norm{ \vect{x}_{j,\ell}^t-\overline{\vect{x}}_{\ell}^{t}} + \norm{\overline{\vect{x}}_{\ell}^{t}-\overline{\vect{x}}_{\ell-1}^{t}} + \norm{\overline{\vect{x}}_{\ell}^{t}-\vect{x}_{j,\ell-1}^{t}} \tag{Triangle inequality}\\
&\leq \frac{C_p}{\ell}  + \norm{\overline{\vect{x}}_{\ell}^{t} - \overline{\vect{x}}_{\ell-1}^{t}} + \frac{C_p}{\ell-1} \tag{By \cref{lmm:avg}} \\
&\leq \frac{C_p}{\ell} + \frac{C_p}{\ell-1} + \frac{AD}{\ell} \tag{By \cref{eq:distance_xbars}}\\
\norm{\vect{x}_{j,\ell}^{t} - \vect{x}_{j,\ell-1}^{t}} &\leq \frac{C_p}{\ell} + \frac{C_p}{\ell-1} + \frac{AD}{\ell} \label{eq:xjl-xjl-1}
\end{align}
\begin{align}
\norm{\vect{d}_{i,\ell}^{t} - \vect{d}_{i,\ell-1}^{t}} &\leq \norm{\vect{d}_{i,\ell}^{t} - \nabla F_{\ell}^{t}} + \norm{\nabla F_{\ell}^{t} - \nabla F_{\ell-1}^{t}} +\norm{\nabla F_{\ell-1}^{t} - \vect{d}_{i,\ell-1}^{t}} \tag{Triangle inequality}\\
&\leq \frac{C_{d}}{\ell} + \norm{\nabla F_{\ell}^{t} - \nabla F_{\ell-1}^{t}} + \frac{C_{d}}{\ell-1} \tag{By \cref{lem:convergence}}\\
&=  \frac{C_{d}}{\ell} + \frac{C_{d}}{\ell-1} + \frac{1}{n} \sum_{j=1}^{n} \norm{\nabla f_{j}^{t}(\vect{x}_{j,\ell}^{t}) - \nabla f_{j}^{t}(\vect{x}_{j,\ell-1}^{t})} \tag{Definition of $\nabla F_{\ell}^{t}$}\\
&\leq \frac{C_{d}}{\ell} + \frac{C_{d}}{\ell-1} + \frac{\beta}{n} \sum_{j=1}^{n} \norm{\vect{x}_{j,\ell}^{t} - \vect{x}_{j,\ell-1}^{t}} \tag{$f_{j}^{t}$ is $\beta$-smooth}\\
&\leq \frac{C_{d}}{\ell} + \frac{C_{d}}{\ell-1} + \beta \left[ \frac{C_p}{\ell} + \frac{C_p}{\ell-1} + \frac{AD}{\ell}\right]  \tag{By \cref{eq:xjl-xjl-1}}\\
&\leq \frac{2C_{d}}{\ell+3} + \frac{2C_{d}}{\ell+3} + \beta \left[ \frac{2C_p}{\ell+3} + \frac{2C_p}{\ell+3} + \frac{2AD}{\ell+3}\right]  \tag{When $\ell \geq 7$}  \\
&= \frac{4C_{d} + 2\beta \left[ 2C_{p} + AD \right]}{\ell+3}\\
&\leq \frac{4C_{d} + 2\beta \left[ 2C_{p} + AD \right]}{(\ell+3)^{\alpha}}
\end{align}
\end{proof}

\begin{remark}
\label{rmk:quantized_fw}
By \cref{lmm:red_var1} and Jensen's inequality, we can deduce the following inequality
\begin{equation*}
    \E{\|\vect{d}^t_{i,\ell} - \Tilde{\vect{a}}^t_{i,\ell} \|} \leq \sqrt{\E{\|\vect{d}^t_{i,\ell} - \Tilde{\vect{a}}^t_{i,\ell} \|^{2}}} \leq \frac{Q^{1/2}}{(\ell + 4)^{1/4}}
\end{equation*}
\end{remark}

\begin{theorem}
\label{thm:stoc:version2}
Let $\mathcal{K}$ be a convex set with diameter $D$. Assume that for every $1 \leq t \leq T$, 
\begin{enumerate}
	\item functions $f^{t}_{i}$ are $\beta$-smooth, i.e. $\nabla f^{t}_{i}$ is $\beta$-Lipschitz,  (so $F^{t}$ is $\beta$-smooth);
	\item $\| \nabla f^{t}_{i}\| \leq G$ (so $\| \nabla F^{t}\| \leq G$);
	\item the gradient estimates are unbiased with bounded variance $\sigma^{2}$, i.e., 
		$\E [\widetilde{\nabla} f^{t}_{i}(\vect{x}^{t}_{i,\ell})] = \nabla f^{t}_{i}(\vect{x}^{t}_{i,\ell})$
		and $\bigl \| \widetilde{\nabla} f^{t}_{i}(\vect{x}^{t}_{i,\ell})] - \nabla f^{t}_{i}(\vect{x}^{t}_{i,\ell}) \bigr \| \leq \sigma$
		for every $1 \leq i \leq n$ and $1 \leq \ell \leq L$;
	\item the gradient estimates are $\widetilde{\beta}$-Lipschitz.
\end{enumerate}
Then, choosing the step-sizes $\eta_\ell = \min \{1, \frac{A}{\ell^{3/4}}\}$ where $A \in \mathbb{R}_+$. For all $1 \leq i \leq n$, 
\begin{align*}
    \max_{\vect{o} \in \mathcal{K}} \E \left[ \frac{1}{T} \sum_{t=1}^T \E_{\vect{x}_i^t}\left[ \langle \nabla F_t\left( \vect{x}_i^t \right), \vect{x}_i^t - \vect{o} \rangle \right] \right]
    &\leq \frac{DG + 2ADQ^{1/2}}{L^{1/4}} + \frac{2AD^2\beta}{L^{3/4}} \\
    &+\left[ \left(\beta D + G\right) + \left(\beta C_p + C_d\right)D \right]  \frac{\log L }{L} + O \left(\mathcal{R}^T \right)
\end{align*}
where $\mathcal{R}^T$ is the regret of online linear minimization oracles,
$C_p$ and $C_d$ are already defined in \Cref{lem:convergence} and $Q = 48\left(\Tilde{\beta}^2 + \beta^2\right) \left(2C_p + AD \right)^2 + 4\sigma^2 + 2B^2$ where $B = 4C_d + 2\beta [2C_p + AD]$ given in \Cref{lmm:red_var1}, \Cref{lem:forQ} and \Cref{clm:B}.

Choosing $L=T$ and oracles as gradient descent or follow-the-perturbed-leader with regret $\mathcal{R}^T =
O\left(T^{-1/2}\right)$, we obtain a convergence rate of $O\left( T^{-1/4} \right).$

\end{theorem}

\begin{proof}
By \cref{tk:exp_gap} in the proof of \cref{thm:gap}, we have:
\begin{align*}
    \E_{\overline{x}^t}\left[ \mathcal{G}^t \right] 
    &\leq \frac{DG}{L^{1/4}A} + \frac{(\beta C_p + C_d) D \log L}{L} + \frac{2\beta AD^2}{L^{3/4}} + \frac{1}{nL}\sum_{\ell=1}^{L} \sum_{i=1}^n \langle \vect{d}_{i,\ell}^t, \vect{v}_{i,\ell}^t - \vect{o}_{\ell}^t \rangle\\
    &\leq \frac{DG}{L^{1/4}A} + \frac{(\beta C_p + C_d) D \log L}{L} + \frac{2\beta AD^2}{L^{3/4}} + \frac{1}{nL}\sum_{\ell=1}^{L} \sum_{i=1}^n \langle \vect{d}_{i,\ell}^t -\widetilde{\vect{a}}_{i,\ell}^t,\vect{v}_{i,\ell}^t - \vect{o}_{\ell}^t \rangle \\
     & \quad + \frac{1}{nL}\sum_{\ell=1}^{L} \sum_{i=1}^n \langle \widetilde{\vect{a}}_{i,\ell}^t,\vect{v}_{i,\ell}^t - \vect{o}_{\ell}^t \rangle\\
    &\leq \frac{DG}{L^{1/4}A} + \frac{(\beta C_p + C_d) D \log L}{L} + \frac{2\beta AD^2}{L^{3/4}} + \frac{1}{nL}\sum_{\ell=1}^{L} \sum_{i=1}^n \|\vect{d}_{i,\ell}^t -\widetilde{\vect{a}}_{i,\ell}^t\| \| \vect{v}_{i,\ell}^t - \vect{o}_{\ell}^t \| \\
     & \quad + \frac{1}{nL}\sum_{\ell=1}^{L} \sum_{i=1}^n \langle \widetilde{\vect{a}}_{i,\ell}^t,\vect{v}_{i,\ell}^t - \vect{o}_{\ell}^t \rangle
     \tag{Cauchy-Schwarz}\\
    &\leq \frac{DG}{L^{1/4}A} + \frac{(\beta C_p + C_d) D \log L}{L} + \frac{2\beta AD^2}{L^{3/4}} + \frac{D}{nL}\sum_{\ell=1}^{L} \sum_{i=1}^n \|\vect{d}_{i,\ell}^t -\widetilde{\vect{a}}_{i,\ell}^t\| \\
     & \quad + \frac{1}{nL}\sum_{\ell=1}^{L} \sum_{i=1}^n \langle \widetilde{\vect{a}}_{i,\ell}^t,\vect{v}_{i,\ell}^t - \vect{o}_{\ell}^t \rangle
     \tag{$\vect{v}_{i,\ell}^t, \vect{o}_{\ell}^t \in \mathcal{K}^2 \Rightarrow \| \vect{v}_{i,\ell}^t -  \vect{o}_{\ell}^t \| \leq D$}
\end{align*}

\begin{align*}
    \E\left[ \frac{1}{T} \sum_{t=1}^T \E_{\overline{x}^t} \left[ \mathcal{G}^t \right] \right]
    &\leq \frac{DG}{L^{1/4}A} + \frac{(\beta C_p + C_d) D \log L}{L} + \frac{2\beta AD^2}{L^{3/4}} + \frac{D}{nLT}\sum_{\ell=1}^{L} \sum_{i=1}^n \sum_{t=1}^T  \E \left[ \|\vect{d}_{i,\ell}^t -\widetilde{\vect{a}}_{i,\ell}^t\| \right] \\
     & \quad + \E \left[ \frac{1}{nLT}\sum_{\ell=1}^{L} \sum_{i=1}^n \sum_{t=1}^T \langle \widetilde{\vect{a}}_{i,\ell}^t,\vect{v}_{i,\ell}^t - \vect{o}_{\ell}^t \rangle \right] \\
     &\leq \frac{DG}{L^{1/4}A} + \frac{(\beta C_p + C_d) D \log L}{L} + \frac{2\beta AD^2}{L^{3/4}} + \frac{Q^{1/2}D}{L}\sum_{\ell=1}^{L} \frac{1}{(\ell+4)^{1/4}} \\
     & \quad + \E \left[ \frac{1}{nLT}\sum_{\ell=1}^{L} \sum_{i=1}^n \sum_{t=1}^T \langle \widetilde{\vect{a}}_{i,\ell}^t,\vect{v}_{i,\ell}^t - \vect{o}_{\ell}^t \rangle \right]
     \tag{By \cref{rmk:quantized_fw}}\\
     &\leq \frac{DG}{L^{1/4}A} + \frac{(\beta C_p + C_d) D \log L}{L} + \frac{2\beta AD^2}{L^{3/4}} + \frac{2Q^{1/2}D}{L^{1/4}}  \\
     & \quad + \E \left[ \frac{1}{nLT}\sum_{\ell=1}^{L} \sum_{i=1}^n \sum_{t=1}^T \langle \widetilde{\vect{a}}_{i,\ell}^t,\vect{v}_{i,\ell}^t - \vect{o}_{\ell}^t \rangle \right]
     \tag{$\sum_{\ell=1}^L \frac{1}{(\ell + 4)^{1/4}}\leq 2 L^{3/4}$}\\
     &\leq \frac{DG}{L^{1/4}A} + \frac{(\beta C_p + C_d) D \log L}{L} + \frac{2\beta AD^2}{L^{3/4}} + \frac{2Q^{1/2}D}{L^{1/4}} + O \left(\mathcal{R}^T\right) \\
     \tag{$\vect{v}^t_{i,\ell}$ are chosen by the online oracles with regret $\mathcal{R}^T$}
\end{align*}
Recall that, 
\begin{align*}
    \E_{\overline{\vect{x}}^t}{ \left[ \mathcal{G}^t \right]}
    &= \E_{\overline{\vect{x}}^t}{\left[ \max_{\vect{o} \in \mathcal{K}} \langle \nabla F_t\left( \overline{\vect{x}}^t \right), \overline{\vect{x}}^t - \vect{o} \rangle \right]} 
\end{align*}
Thecrefore by \cref{lmm:final_step}, 
\begin{align*}
    \E \left[ \frac{1}{T} \sum_{t=1}^T \E_{\vect{x}_i^t}\left[ \max_{\vect{o} \in \mathcal{K}} \langle \nabla F_t\left( \vect{x}_i^t \right), \vect{x}_i^t - \vect{o} \rangle \right] \right]
    &\leq \E \left[ \frac{1}{T} \sum_{t=1}^T  \E_{\overline{\vect{x}}^t}{\left[ \max_{\vect{o} \in \mathcal{K}} \langle \nabla F_t\left( \overline{\vect{x}}^t \right), \overline{\vect{x}}^t - \vect{o} \rangle \right]} \right] \\
    & \quad + \frac{(\beta D + G)C_p \log L}{L} \\
    &\leq \frac{DG + 2ADQ^{1/2}}{L^{1/4}} + \frac{2AD^2\beta}{L^{3/4}} \\
    &+\left[ \left(\beta D + G\right) + \left(\beta C_p + C_d\right)D \right]  \frac{\log L }{L} + O \left(\mathcal{R}^T \right) 
\end{align*}
Since $\max$ is a convex function, the theorem follows by applying Jensen's inequality on the left-hand side of the above equation.
\end{proof}

\end{document}